\documentclass{article} 
\usepackage{iclr2025_conference,times}

\usepackage{amsmath,amsfonts,bm}

\def\eqref#1{equation~\ref{#1}}

\def\1{\bm{1}}

\DeclareMathAlphabet{\mathsfit}{\encodingdefault}{\sfdefault}{m}{sl}
\SetMathAlphabet{\mathsfit}{bold}{\encodingdefault}{\sfdefault}{bx}{n}

\DeclareMathOperator*{\argmin}{arg\,min}

\newcommand{\TCS}{\widehat{TC}}

\usepackage{url}

\usepackage[linesnumbered,ruled,vlined]{algorithm2e}

\SetCommentSty{mycommfont}

\usepackage{soul}
\usepackage{comment}
\usepackage{xcolor}

\usepackage{pgfplots}
\usepackage{booktabs}
\usepackage{adjustbox}
\usepackage{caption}

\usepackage{amsmath}
\usepackage{amssymb}
\usepackage{mathtools}
\usepackage{amsthm}
\usepackage{amsfonts}
\usepackage{graphicx}
\usepackage{lipsum}
\usepackage{thm-restate}

\usepackage{hyperref}
\usepackage{soul}
\usepackage{adjustbox}
\usepackage{wrapfig}
\usepackage{pgfplots}
\pgfplotsset{compat=1.18} 
\usepackage{pgfplotstable} 

\usepackage[capitalise]{cleveref}

\theoremstyle{plain}
\newtheorem{theorem}{Theorem}[section]

\theoremstyle{definition}

\theoremstyle{remark}
\newtheorem{remark}[theorem]{Remark}

\DeclareMathOperator{\TC}{TC}

\DeclareMathOperator{\PGM}{PGM}
\DeclareMathOperator{\OT}{OT}

\DeclareMathOperator{\OPT}{OPT}
\DeclareMathOperator{\WOPT}{WOPT}

\usepackage[textsize=tiny]{todonotes}

\usepackage[utf8]{inputenc} 
\usepackage[T1]{fontenc}    
\usepackage{url}            
\usepackage{booktabs}       
\usepackage{nicefrac}       
\usepackage{microtype}      
\usepackage{xcolor}         
\usepackage[normalem]{ulem}

\definecolor{blue(ryb)}{rgb}{0.01, 0.28, 1.0}
\definecolor{blue(pigment)}{rgb}{0.2, 0.2, 0.6}

\definecolor{camouflagegreen}{rgb}{0.47, 0.53, 0.42}
\definecolor{camel}{rgb}{0.76, 0.6, 0.42}
\definecolor{bluegray}{rgb}{0.4, 0.6, 0.8}
\definecolor{blue(pigment)}{rgb}{0.2, 0.2, 0.6}
\definecolor{cambridgeblue}{rgb}{0.64, 0.76, 0.68}
\definecolor{color1}{RGB}{166,206,227}
\definecolor{color2}{RGB}{31,120,180}
\definecolor{color3}{rgb}{0.7, 0.75, 0.71}
\definecolor{mismatchColor}{RGB}{165,42,42}  
\definecolor{partialityColor}{RGB}{0,128,128}

\title{Learning Partial Graph Matching via Optimal Partial Transport}

\author{Gathika Ratnayaka\textsuperscript{\rm 1}, James Nichols\textsuperscript{\rm 2}, \& Qing Wang\textsuperscript{\rm 1},  \\
    \textsuperscript{\rm 1} School of Computing, Australian National University, Australia\\
    \textsuperscript{\rm 2}
Biological Data Science Institute, Australian National University, Australia\\
\texttt{\{gathika.ratnayaka,james.nichols,qing.wang\}@anu.edu.au}}

\begin{document}

\maketitle

\begin{abstract}

Partial graph matching extends traditional graph matching by allowing some nodes to remain unmatched, enabling applications in more complex scenarios. However, this flexibility introduces additional complexity, as both the subset of nodes to match and the optimal mapping must be determined. While recent studies have explored deep learning techniques for partial graph matching, a significant limitation remains: the absence of an optimization objective that fully captures the problem’s intrinsic nature while enabling efficient solutions. In this paper, we propose a novel optimization framework for partial graph matching, inspired by optimal partial transport. Our approach formulates an objective that enables partial assignments while incorporating matching biases, using weighted total variation as the divergence function to guarantee optimal partial assignments. Our method can achieve efficient, exact solutions within cubic worst case time complexity. Our contributions are threefold: (i) we introduce a novel optimization objective that balances matched and unmatched nodes; (ii) we establish a connection between partial graph matching and linear sum assignment problem, enabling efficient solutions; (iii) we propose a deep graph matching architecture with a novel partial matching loss, providing an end-to-end solution. The empirical evaluations on standard graph matching benchmarks demonstrate the efficacy of the proposed approach.
\end{abstract}

\section{Introduction}
\label{sec:introduction} 

Graph matching is a fundamental problem in network analysis, aiming to establish one-to-one correspondences between nodes in two graphs based on a defined objective. It has broad applications across various fields, including computer vision~\citep{sun2020survey}, bioinformatics~\citep{zaslavskiy2009global}, and social network analysis~\citep{zhang2019graph}, where it is used to solve complex real-world problems. Traditional graph matching assumes a bijective mapping between the nodes of two graphs or a total injective mapping from the smaller graph to the larger one. However, these assumptions often restrict its applicability in more complex, real-world scenarios. \emph{Partial graph matching}, a generalized version of the graph matching problem, addresses these limitations by allowing some nodes in both graphs to remain unmatched~\citep{wang2023deep,jiang2022graph}. It seeks an optimal partial assignment, an injective partial function between node sets, thereby expanding the practical utility of graph matching. This approach is particularly useful in applications where not all nodes have meaningful counterparts. For example, in image keypoint matching, not all keypoints in two given images have correspondences~\citep{jiang2022graph}, and in biological networks, certain proteins may lack direct counterparts in other species~\citep{zaslavskiy2010graph}.
 
Solving partial graph matching is inherently more complex than traditional graph matching due to the additional challenge of determining both the subset of nodes to match and the optimal mapping itself. Graph matching is generally framed as a combinatorial optimization problem, where the choice of optimization objective is crucial. Early works formulated it as a Quadratic Assignment Problem (QAP), which is NP-hard~\citep{koopmans1957assignment, lawler1963quadratic}. Recent approaches ~\citep{wang2021neural, wang2020combinatorial, rolinek2020deep, fey2020deep} use neural networks to learn node representations and derive a cross-graph node-to-node cost matrix. This allows graph matching to be framed as a linear sum assignment problem~\citep{burkard2012assignment, yu2019learning}, assuming a bijective mapping between two node sets of equal size. Even if two graphs have an unequal number of nodes, the problem can still be framed as a linear sum assignment by assuming an injective mapping from the smaller graph to the larger one~\citep{bonneel2019spot, wang2019learning}. A key advantage of this approach is that it can be efficiently solved with cubic worst-case time complexity using the Hungarian algorithm~\citep{kuhn1955hungarian}. However, Partial graph matching does not assume a total mapping, adding an extra layer of complexity. It requires selecting the subset of nodes to be matched while determining the mapping, making the problem more challenging. Consequently, traditional linear assignment methods like linear sum or \textit{k} assignment cannot be directly applied.

Despite its challenging nature, recent studies have attempted to address the partial graph matching problem using deep learning techniques~\citep{jiang2022graph, wang2023deep}. One approach~\citep{jiang2022graph} frames the problem as an Integer Linear Programming (ILP) task with dummy nodes. However, ILPs rely on branch and bound algorithms, which lack polynomial worst-case time complexity, and the use of dummy nodes can hinder node representation learning~\citep{wang2023deep}. Another approach~\citep{wang2023deep} estimates the number of matchings (\textit{k}) between two graphs, solving it as a $k$-assignment problem~\citep{burkard2012assignment} using GreedyTopK algorithm~\citep{wang2023deep}. However, determining the optimal \textit{k} is difficult and requires separate neural modules, leading to multiple training stages. These limitations highlight the need for an efficient optimization objective that captures the inherent nature of partial graph matching.

\paragraph{Present work}We address limitations in partial graph matching by proposing a novel optimization objective that finds an optimal partial mapping between two node sets while identifying which nodes should be matched. We first observe that, despite different frameworks, partial graph matching and optimal partial transport~\citep{bai2023sliced,sejourne2023unbalanced} share the same goal: finding an optimal partial correspondence between two sets of elements (nodes in graph matching and mass points in optimal transport) to minimize a cost function. Specifically, partial graph matching seeks an optimal mapping where not all nodes are matched, while optimal partial transport identifies a plan where not all mass is transported. Drawing from this connection, we explore how the principles of optimal partial transport can be leveraged to formulate an optimization objective for partial graph matching. However, a key difference emerges: partial graph matching requires the mapping between two sets to be an injective partial function (i.e., a partial assignment), whereas optimal partial transport does not. In real-world applications, some nodes may need to be prioritized for matching based on data, requiring the optimization objective for partial graph matching to account for these biases. 

In optimal partial transport formulations,  divergence functions are typically used to penalize untransported mass, which in turn determines the amount of mass that will be transported. We theoretically show that using weighted total variation as the divergence function allows for an optimal partial transport objective that guarantees the existence of a partial assignment while incorporating node matching bias.  Based on these theoretical properties, we define a new optimization problem for partial graph matching. We further demonstrate that the partial matching can solved by embedding it in a linear assignment problem. This enables the use of the Hungarian algorithm to exactly solve the partial graph matching problem.

In summary, we make the following contributions in this work:

\begin{itemize}
    \item We define a new optimization problem for partial graph matching, inspired by optimal partial transport. By incorporating a cost matrix as well as matching biases, this formulation provides a robust optimization objective that carefully balances the selection of matched and unmatched nodes. 
    \item To solve the proposed optimization problem efficiently, we explore the underlying structure of its solution spaces. This reveals a notable embedding of the partial graph matching problem in a linear sum assignment problem. Furthermore, the solutions of the latter assignment problem can be mapped efficiently to solutions of the partial matching problem, which are themselves optimal. This allows us to solve the partial graph matching problem exactly using the Hungarian algorithm.
    \item 
    Building on the theoretical insights of our proposed optimization objective, we introduce a deep graph matching architecture that embeds feature and structural properties into the cross-graph node-to-node cost matrix and matching biases. This architecture provides an end-to-end solution for the partial graph matching problem, incorporating a novel loss function called \emph{partial matching loss}.    
\end{itemize}

We conduct experiments to empirically validate our proposed approach on partial graph matching benchmarks. The results demonstrate the efficacy and efficiency of the proposed approach.

\section{Related Work}
\label{sec:related}

\subsection{Deep graph matching}
Several graph matching methods leverage neural networks to learn matching-aware node embeddings~\citep{jiang2022graph, rolinek2020deep, wang2020combinatorial, yu2019learning, gao2021deep}. These methods integrate cross-graph node affinity with feature and structural data to learn node embeddings. The learned embeddings are then used to derive cross-graph node-to-node affinities. Then, the soft correspondences between nodes are typically obtained by applying Sinkhorn normalization on the cost-graph node-to-node affinity matrix. Once the soft correspondence matrix is obtained, it is projected  into the space of permutation like binary matrices to achieve one-to-one node correspondences using algorithms such as the Hungarian algorithm~\citep{kuhn1955hungarian} and Stable Matching algorithm~\citep{ratnayaka2023contrastive}, assuming a total injective mapping from the smaller graph to the larger graph. 

Optimal transport techniques have also been discussed to solve the graph matching problem ~\citep{xu2019gromov}, where each node in the source graph is matched to a node in the target graph. Moreover, optimal partial transport techniques have been applied to subgraph matching ~\citep{pan2024subgraph}, where a preset fraction of mass (corresponding to a fixed number of nodes) from a smaller graph is matched to nodes in a larger graph. This approach constrains the matching process by specifying in advance how many nodes from the smaller graph must be matched. However, these approaches are not directly applicable to partial graph matching, which requires identifying corresponding nodes between two graphs without any prior assumptions about the number of nodes to be matched. Moreover,spectral methods have been successfully applied to both graph matching ~\citep{wang2019functional} and also to compactly encode maps between graphs and
subgraphs ~\citep{pegoraro2022spectral}. While these problems share similarities in attempting to find correspondences at node or graph level, they differ fundamentally from partial graph matching in their matching objectives and constraints.

So far, only a few deep learning approaches have addressed partial graph matching. \cite{jiang2022graph} introduces dummy nodes to bypass explicit match number estimation, framing the problem as an Integer Linear Programming (ILP) task. However, ILP with branch and bound suffers from high time complexity, and the use of dummy nodes can distort node representations by implying higher similarity with unmatched nodes. \cite{wang2023deep} attempts to solve partial graph matching as a k-assignment problem. However, as the match count (k) between two graphs is not known in a partial graph matching problem, their approach to handle partial matchings include two steps, first estimating the match count (k) by using a separate neural module and solving an entropic optimal transport problem and then using the estimated k value to solve partial graph matching as a k-assignment problem. This two step approach increase computational complexity and error propagation.
Another line of work ~\citep{nurlanov2023universe,jiang2022learning} to perform partial graph matching have explored universe graph representation learning, where a "universe" graph is constructed for each object class in keypoint-based image analysis. However, these methods require prior knowledge specific to each class, such as the exact number of distinct keypoints (nodes) in a class, making them difficult to generalize to graphs or images from previously unseen classes. In contrast, our work addresses a more general form of partial graph matching problem that operates without any class-specific assumptions, allowing it to work with arbitrary graph structures. 

\subsection{Optimal Transport and the assignment problem}

Assignment problems, especially the linear sum assignment problem, can be modeled as an Optimal transport problem by considering assignment costs as transportation costs between two discrete distributions~\citep{burkard2012assignment}. The Hungarian algorithm~\citep{kuhn1955hungarian} which is a well-known methods for solving the linear sum assignment problem, operates with a worst-case time complexity of $O(n^3)$, and is widely used in many real world applications~\citep{munkres1957algorithms,yu2019learning}.

In contrast to the linear sum assignment problem, partial graph matching can be considered an assignment problem where elements in both sets can remain unassigned. This characteristic makes conventional optimal transport methods unsuitable for solving partial assignment problems. While efforts have been made to adapt optimal partial transport for partial assignments~\citep{bai2023sliced, bonneel2019spot}, these methods have limitations. \citep{bai2023sliced} propose an optimal partial transport formulation that can guarantee an optimal solution inducing a partial assignment between two sets, but their algorithm is restricted to one-dimensional data with convex cost metrics, whereas partial graph matching typically involves higher-dimensional data and requires more flexible cost functions. Moreover, their formulation does not account for matching bias of elements, which is crucial in partial graph matching, especially in data-driven approaches where certain elements should be prioritized for assignment. Efficient algorithms ($\leq(O(n^3))$) for solving partial assignments in higher-dimensional spaces with general cost functions have not yet been developed, leaving a gap in the literature.

\section{Background}
\label{sec:bac}

\paragraph{Notations}  We denote the set of source elements as $W_S$ and the set of target elements as $W_T$. We also use $\| \cdot \|_1$ to refer the $L^1$ norm, $\langle , \rangle_{F}$ to refer the Frobenius inner product, and $\mathbf{1}_n$ to denote a column vector of ones with $n$ elements. For any given matrix $\pi \in \mathbb{R}^{m \times n}$, $\pi_{1}$ and $\pi_{2}$ denotes marginals of $\pi$ where $\pi_{1}=\pi\mathbf{1}_n$ and $\pi_{2}=\pi^{\intercal  }\mathbf{1}_m$

\paragraph{Optimal Transport (Balanced)}

Let \( \mu \in \mathbb{R}_{\geq 0}^n \) and \( \nu \in \mathbb{R}_{\geq 0}^m \) be two non-negative vectors representing distributions with equal total mass, such that \( \|\mu\|_1 = \|\nu\|_1 \). We are also given a \emph{cost} matrix \( C \in \mathbb{R}^{m \times n} \), where \( C_{ij} \) denotes the cost of transporting a unit of mass from location \( i \) in the source set \( W_S \) to location \( j \) in the target set \( W_T \). 
The \emph{(balanced) optimal transport problem} considers all possible transport plans $    \Pi(\mu,\nu)= \{ \pi \in \mathbb{R}_{\geq 0}^{m\times n} | \pi \mathbf{1}_n=\mu, \pi^{\intercal  }\mathbf{1}_m =\nu \}
$ and is defined as,
\begin{equation}
    \label{eq:BOT}
   \OT(\mu,\nu) \coloneqq \min_{\pi \in \Pi(\mu,\nu)}\;\langle\,\pi,C \rangle_{F}.
\end{equation}
The marginal conditions $\pi\mathbf{1}_n=\mu$ and $\pi^{\intercal}\mathbf{1}_m =\nu $ impose the mass conservation constraint, ensuring that the total mass is transported from one distribution to the other.

\paragraph{Optimal Partial Transport}
In the optimal partial transport problem, the mass conservation constraint is relaxed, allowing partial mass transportation between distributions. 
While some studies~\citep{chapel2020partial,figalli2010optimal} address predefined amounts of partial mass,
we consider the optimal partial transportation problem where the amount of partial mass being transported is determined by the optimization objective. 

Let $\Pi_{\leq}(\mu,\nu)$ be the set of admissible {\em partial} transport plans  defined as
\begin{equation} 
\label{eq:partial_plans}
    \Pi_{\leq}(\mu,\nu)= \{ \pi \in \mathbb{R}_{\geq 0}^{m\times n} | \pi\mathbf{1}_n \leq \mu, \pi^{\intercal  }\mathbf{1}_m \leq \nu \}.
\end{equation} 
The \emph{optimal partial transport problem} is usually defined as
\begin{align}
\label{eq:POT}
\begin{split}
    \OPT_\rho(\mu,\nu) \coloneqq \min_{\pi\in \Pi_{\leq}(\mu,\nu)}\langle\,\pi,C \rangle_{F} + \rho D(\pi_{1}|\mu)+  \rho D(\pi_{2}|\nu).
\end{split}
\end{align}
Here, $\rho>0$, which is also termed as the unbalancedness parameter, is used to control the tolerance for destroying or creating mass, and $D(\cdot|\cdot)$ represents a divergence between two discrete measures. The term $\langle\pi,C \rangle_F $ captures the cost of transporting the mass by $\pi$ while $ \rho D(\pi_{1}|\mu)+ \rho D(\pi_{2}|\nu)$ accounts for the mass that has not been transported.

In the optimal transport literature, Total Variation (TV) and Kullback-Leibler (KL) divergence are commonly used divergence functions~\citep{sejourne2023unbalanced}. 
Using TV as the divergence function in optimal partial transport is known to identify zero entries in the optimal plan based on the cost matrix ~\citep{bai2023sliced,sejourne2023unbalanced}.

\section{Partial Graph Matching Problem}\label{sec:problem}

We represent a graph as $\mathcal{G} = (V, E)$, where $V$ is the set of nodes and $E$ is the set of edges. Let $\mathcal{G}_S = (V_S, E_S)$ be the source graph and $\mathcal{G}_T = (V_T, E_T)$ be the target graph for matching, with $|V_S|=m$ and $|V_T|=n$. Without loss of generality, we can assume $m \leq n$. 

In this work, we tackle the partial graph matching problem through the lens of optimal partial transport, which offers a flexible and efficient framework for dealing with graphs of different sizes and structures. 
To this end, we define an objective function that is well suited to our partial graph matching task. We first state this objective as a general optimal transport problem, then adapt it to our setting of graph matching.

We represent $V_S$ and $V_T$ as mass vectors $\mu \in \mathbb{R}^{m}$ and $\nu \in \mathbb{R}^{m}$ respectively. We also assume a cost matrix $C \in \mathbb{R}^{m \times n}$, where $C_{ij}$ indicates the cost of moving a unit mass from $i \in V_S$ to $j \in V_T$. We take inspiration from the use of the objective function defined in \cref{eq:POT}, and adapt it to our purposes. In particular, we use a weighted {\em total variation} (TV) divergence, which both ensures sparse optimal transport plans,  
and furthermore allows us to incorporate a {\em matching bias} to each node in the source and target graphs. 
The matching bias weights are given by two vectors, $\alpha \in \mathbb{R}_{\geq 0}^m $ for nodes in $V_S$ and $ \beta \in \mathbb{R}_{\geq 0}^n $ for nodes in $V_T$,
finally resulting in a objective function given by
\begin{align}
    \label{eq:UGM-tc}
    \TC(\pi; C, \alpha, \beta) \coloneqq \langle \pi, C \rangle_F + \rho\left( \langle \alpha, \mu - \pi_1 \rangle 
    + \langle \beta, \nu - \pi_2 \rangle \right).
\end{align}

As we consider marginal constraints $\pi\mathbf{1}_n \leq \mu$ and  $\pi^{\intercal  }\mathbf{1}_m \leq \nu$, the weighted total variation divergences can be written as  $\langle \alpha, \mu - \pi_1 \rangle$, and $\langle \beta, \nu - \pi_2 \rangle$. In the general transport setting our set of feasible partial plans would be $\Pi_{\leq}(\mu, \nu)$ as defined as in \cref{eq:partial_plans}, and the optimal partial transport problem with weighted total variation as the divergence function then is
\begin{align}
    \label{eq:UGM}
\WOPT_{\rho}(\mu, \nu) \coloneqq \min_{\pi \in \Pi_{\leq}(\mu, \nu)} 
\TC(\pi;  C, \alpha, \beta).
\end{align}

In our setting however, we require optimal partial {\em assignments} rather than plans, meaning that each node in one graph is matched to a node in the other graph or no node at all. 
We can define the set $\mathcal{M}$ of all $m \times n$ binary matrices representing possible partial assignments between $V_S$ and $V_T$ as follows
\[
\mathcal{M} \coloneqq \{\pi \in \{0,1\}^{m \times n} \mid \forall i \in V_S, \sum_{j=1}^{n} \pi_{ij} \leq 1 \text{ and } \forall j \in V_T, \sum_{i=1}^{m} \pi_{ij} \leq 1\}.
\]
For any $\pi \in \mathcal{M}$, $\pi_{ij} = 1$ if and only if node $i \in V_S$ is matched with node $j \in V_T$. 
Given the graphs $\mathcal{G}_S$ and $\mathcal{G}_T$, the partial graph matching problem seeks to find an optimal $\pi \in \mathcal{M}$ with respect to the objective given in \cref{eq:UGM-tc}, that is we are solving the 
\begin{align}
    \label{eq:PGM}
\PGM_\rho \coloneqq \min_{\pi \in \mathcal{M}} 
\TC(\pi;  C, \alpha, \beta).
\end{align}
Note that the partial assignments form a subset of the partial transport plans between $\mathbf{1}_m$ and $\mathbf{1}_n$, that is $\mathcal{M} \subset \Pi_{\le}(\mathbf{1}_m, \mathbf{1}_n)$, and in fact we have that
$\mathcal{M} = \Pi_{\le}(\mathbf{1}_m, \mathbf{1}_n) \cap \{ 0, 1\}^{m \times n}$. 

\section{Partial Graph Matching Solutions}
\label{sec:pgms}

In this section we demonstrate that partial assignments solutions of $\PGM_\rho$  (i.e. that sit in $\mathcal{M}$) do indeed occur in the solution set of $\WOPT_\rho(\mathbf{1}_m, \mathbf{1}_n)$, which we can consider to be a {\em relaxed} version of $\PGM_\rho$. We further note that in general solutions of $\WOPT_\rho(\mu, \nu)$ are not necessarily partial assignments as the marginal constraints might rule out solutions in $\mathcal{M}$.

First we consider the following theorem, which shows that our choice of weighted TV defined in \cref{eq:UGM-tc} can distinguish between feasible and non-feasible assignments. 

\begin{restatable}[Infeasible Assignments]{theorem}{lemmaopt} 
    \label{lem:1}
    Let $\pi^\star \in \argmin_{\Pi_{\leq}(\mu,\nu)}\TC(\pi; C,\alpha,\beta)$ be any optimal solution of \cref{eq:UGM}. For any $1 \le i \le m$ and $1 \le j \le n$ we have that $C_{ij}>\rho(\alpha_i+\beta_j)
    \implies 
    \pi^\star_{ij} = 0$. 
\end{restatable}

From this theorem,  we define $\rho(\alpha_{i}+\beta_{j})$ as the \textit{feasibility threshold} for $i$ and $j$. If the cost $C_{ij}$ exceeds this threshold, no mass will be transported between $i$ and $j$. When considering graph matching, this means that no assignment will occur between $i \in V_S$ and $j \in V_T$.

Now we consider the unit-weighted case, that is when $\mu = \mathbf{1}_m$ and $\nu = \mathbf{1}_n$.
The following theorem demonstrates that there is at least one solution of 
\cref{eq:UGM} that is a partial assignment between $V_S$ and $V_T$, i.e.~is a solution of \cref{eq:PGM}.
\begin{restatable}[Existence of Optimal Solution]{theorem}{lemmaoto} 
    \label{lem:2}
    If $\mu = \mathbf{1}_m$, $\nu = \mathbf{1}_n$ there exists an optimal plan $\pi^\star \in \argmin_{\Pi_{\leq}(\mu,\nu)}\TC(\pi; C, \alpha, \beta)$ such that $\pi^\star \in \mathcal{M}$.
\end{restatable}

For the remainder of the paper we consider \(\mu = \mathbf{1}_m\), \(\nu = \mathbf{1}_n\), and to distinguish our objective function in this case, we write $\TCS$ to denote the total cost in this unit-weight setting, that is

\begin{equation}
\label{eq:TC}
    \TCS(\pi; C, \alpha, \beta) 
    \coloneqq 
    \langle \pi, C \rangle_F + \rho\left( \langle \alpha, \mathbf{1}_m - \pi_1 \rangle 
    + \langle \beta, \mathbf{1}_n - \pi_2 \rangle \right).
\end{equation}

From ~\cref{lem:2}, when \(\TCS(\pi; C, \alpha, \beta)\) is the objective function of ~\cref{eq:UGM}, there exists an optimal plan \(\pi \in \mathcal{M}\) that ensures each node is either matched with at most one element or left unmatched, inducing a valid partial matching between \(V_S\) and \(V_T\).
Moreover, based on ~\cref{lem:1}, if \(\alpha_{p} > \alpha_{q}\) for two source elements \(p\) and \(q\), then \(p\) has a higher feasibility threshold with each target element compared to \(q\), giving \(p\) a higher chance of being assigned to a target element.

\section{Solving the Partial Graph Matching Problem}
\label{sec:solvingPGM}

Although \cref{lem:2} shows that there is always an optimal transport plan that induces a valid solution to the partial graph matching problem, it does not specify how to find such a plan. In this section, we address this challenge. 
Our main result demonstrates that it is possible to derive a linear sum assignment problem for which a closed-form mapping exists from any of its given solution to a solution of the optimal partial transport problem defined in \cref{eq:UGM}.
The significance of this result is that it allows us to adapt the celebrate Hungarian algorithm to 
solve the partial graph matching problem with cubic worst-case time complexity algorithm. 

Given two sets of elements with equal cardinality, the linear sum assignment problem aims to find a bijective assignment that minimizes the total cost of assignment ~\citep{burkard2012assignment}. In order to demonstrate a connection between the partial graph matching problem and the linear sum assignment problem, we first create a set $V_{D}$ by appending $(n-m)$ dummy elements to $V_S$, thus making $|V_D| = |V_T| = n$.  Let $\alpha_* > 0 $ be any value such that $ \alpha_*>\smash{\displaystyle\max_{1 \leq i \leq m}\alpha_i}$. We define a cost matrix $\overline{C} \in \mathbb{R}^{n \times n}$ s.t.,
\begin{equation} \label{eq:C_bar}
    \overline{C}_{ij}= 
\begin{cases}
    C_{ij},& \text{if } 1 \le i \leq m \text{ and }C_{ij} \leq \rho(\alpha_i+\beta_j) \\
    \rho(\alpha_i+\beta_j) ,& \text{if }  1 \le i \leq m \text{ and }C_{ij} > \rho(\alpha_i+\beta_j) \\
    \rho(\alpha_*+\beta_j) ,& \text{if }  m < i \le n.
\end{cases}
\end{equation}

Let $\mathcal{P}_n= \{ \pi \in \{0,1\}^{n\times n} \,|\, \pi\mathbf{1}_n = \mathbf{1}_n,  \pi^{\intercal}\mathbf{1}_n = \mathbf{1}_n \}$ denote the set of $n \times n$ permutation matrices. Then, we consider the following linear sum assignment problem,
\begin{equation}\label{eq:LAP}
    \text{LAP}(V_D, V_T, \overline{C}) \coloneqq \min_{\overline{\pi} \in \mathcal{P}} \langle \overline{\pi}, \overline{C} \rangle_F.
\end{equation}

We first establish an equivalence between the objective functions of the optimal partial transport problem in \cref{eq:UGM} and the linear sum assignment problem in \cref{eq:LAP}.    

To this end we define a closed form mapping that will obtain valid solutions to the original partial graph matching problem \cref{eq:PGM} from solutions of the linear sum assignment problem \cref{eq:LAP}. 
We write $h : \mathbb{R}^{n \times n} \to \mathbb{R}^{m \times n}$ to denote the mapping which for $1 \le i \le m$ and $1 \le j \le n$ is given by
\begin{equation} \label{eq:h}
    [h(\overline \pi)]_{ij} = 
\begin{cases}
    \overline{\pi}_{ij},& \text{if } C_{ij} \leq \rho(\alpha_i+\beta_j), \\
    0 & \text{otherwise}.
\end{cases}
\end{equation}
\begin{restatable}[Equivalence]{lemma}{theoremeqLAP} 
    \label{theorem lap2}
Given a cost matrix $C$ and the weights $\alpha$ and $\beta$, any permutation matrix $\overline \pi \in \mathcal{P}_n$ will have $h(\overline \pi) \in \mathcal{M}$ and $h(\overline \pi)$ satisfies
$\langle \overline{\pi}, \overline{C} \rangle_F = \TCS(h(\overline\pi); C, \alpha, \beta) + \rho (n - m)\alpha_*$.
\end{restatable}

The following theorem effectively states that $h(\overline\pi^\star)$ will be a valid solution for our problem of interest.

\begin{restatable}[Optimal Solution]{theorem}{theoremeqLAPthree}
\label{theorem:lap2}
If \(\overline{\pi}^\star \in \argmin_{\overline{\pi} \in \mathcal{P}_n} \langle \overline{\pi}, \overline{C} \rangle_F\), 
then \(\pi^\star = h(\overline{\pi}^\star)\) is a solution of partial graph matching problem, that is
 \(\pi^\star \in \argmin_{\pi \in \mathcal{M}} \TCS(\pi; C, \alpha, \beta)\).
\end{restatable}

The proofs of all the theorems and lemmas including ~\cref{theorem:lap2} are provided in the appendix. 

\smallskip
\begin{remark}
It is known that the linear sum assignment problem can be solved using the Hungarian algorithm, which has a cubic worst-case time complexity ~\citep{kuhn1955hungarian}.  Therefore, the worst case time complexity of solving the partial graph matching problem is $O(n^3)$. 
\end{remark}

\section{Partial Graph matching Architecture}

\label{sec:opgm}

 In this section, we present an end-to-end learning architecture (~\cref{fig:1}) for partial graph matching, building upon the theoretical properties discussed earlier. This architecture learns both the cost matrix $C$ and the matching biases of nodes ($\alpha$ and $\beta$). We also introduce a novel loss function, \textit{partial matching loss}, designed to optimize performance in partial graph matching.

\begin{figure*}
\centering
\includegraphics[width=.95\textwidth]{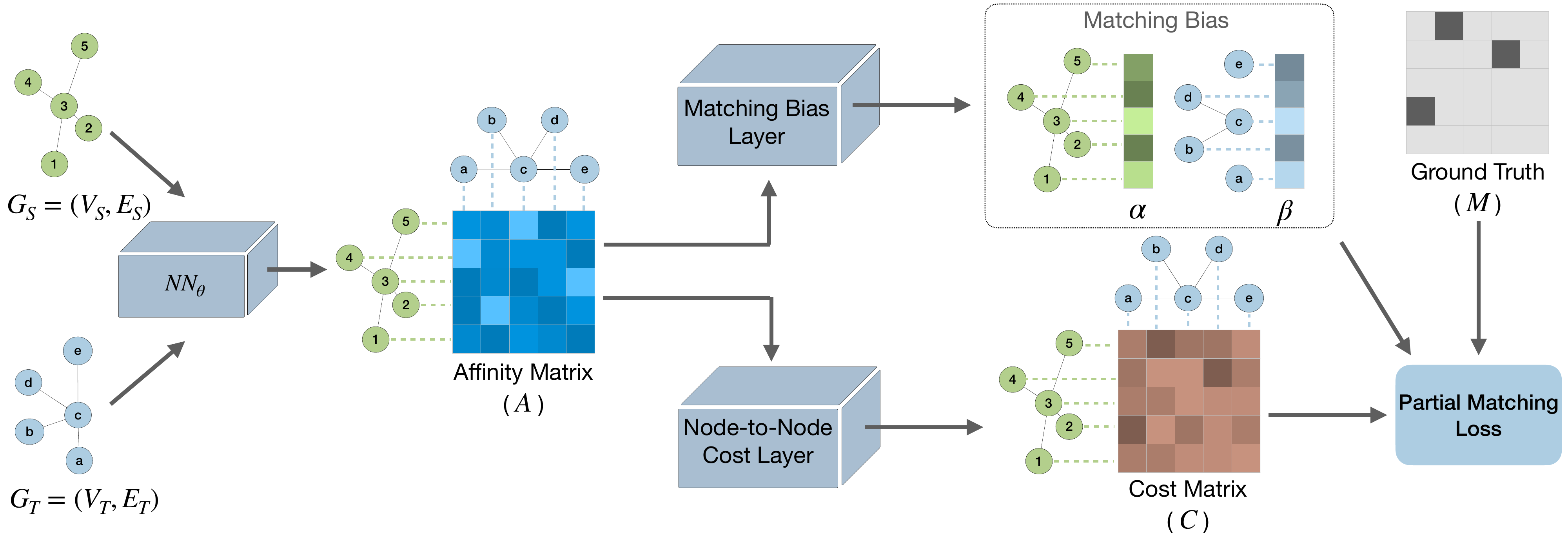}
\caption{
An overview of our end-to-end learning architecture for partial graph matching. Given two input graphs \( G_S \) and \( G_T \), the neural network \( NN_{\theta} \) generates the cross-graph node-to-node affinity matrix \( A \). The affinity matrix \( A \) is then used to learn the matching bias values \( \alpha \), \( \beta \), and also the cost matrix $C$.  The partial matching loss is computed by taking into account the ground truth, \( C \), \( \alpha \), and \( \beta \).
}
\label{fig:1}
\end{figure*}

\paragraph{Graph Affinity Encoding}
Graph affinity encoding uses a neural network $\textsc{nn}_{\theta}$ to transform geometric affinities between nodes into node embeddings, leveraging their features and structural information. These embeddings are then used to construct a matrix representing the cross-graph affinity between nodes in graphs \( \mathcal{G}_S \) and \( \mathcal{G}_T \). Specifically, given two input graphs \( \mathcal{G}_S \) and \( \mathcal{G}_T \), the neural network \( \textsc{nn}_{\theta}: \mathbb{G} \times \mathbb{G} \rightarrow \mathbb{R}^{m \times n} \), parameterized by \( \theta \), returns a cross-graph node-to-node affinity matrix \( A \in \mathbb{R}^{m \times n}\). Essentially, the neural network \( \textsc{nn}_{\theta} \) first applies an embedding function \( f_{emb}: \mathbb{G} \rightarrow \mathbb{R}^{m \times d} \)  to generate node embeddings \( f_{emb}(\mathcal{G}_S) \in \mathbb{R}^{m \times d} \) and \( f_{emb}(\mathcal{G}_T) \in \mathbb{R}^{n \times d} \) for the graphs \( \mathcal{G}_S \) and \( \mathcal{G}_T \), respectively. Then, the affinity function \( f_{aff}: \mathbb{R}^{m \times d} \times \mathbb{R}^{n \times d} \rightarrow \mathbb{R}^{m \times n} \) combines these embeddings to compute a cross-graph node-to-node affinity matrix such that 
\[
\textsc{nn}_{\theta}(\mathcal{G}_S, \mathcal{G}_T) = f_{aff}\left(f_{emb}(\mathcal{G}_S), f_{emb}(\mathcal{G}_T)\right).
\]

\paragraph{Matching Biases and Cost} The matching biases \( \alpha \) and \( \beta \) are calculated based on the cross-graph node-to-node affinity matrix \( A \).  Let  $A^+ \in \mathbb{R}_{\geq 0}^{m \times n}$ be defined as  $A^{+}_{ij} = \max\{A_{ij}, 0\}$. For each node \( i \in V_S \), define \( r_{i} =\max_{1 \leq j \leq n} A^{+}_{ij}\). Similarly, for each node \( j \in V_T \), define \( r_{j} =\max_{1 \leq i \leq m} A^{+}_{ij} \). The matching bias is then calculated as \( \alpha_{i} = 2 \times (\sigma(w_{rs} \times r_{i})-0.5)\) for each \( i \in V_S \) and \( \beta_{j} = 2 \times (\sigma(w_{rs} \times r_{j}) - 0.5)\) for each \( j \in V_T \), where \( \sigma(\cdot) \) is the sigmoid function, and \( w_{rs} \geq 0\) is a learnable parameter. A value of \( \alpha_{i} \) or \(\beta_{j} \) closer to 1 indicates a higher chance of node \( i \) or $j$ being matched, while a value closer to 0 indicates a lower chance. Similarly, for \( \beta_{j} \), a value closer to 1 suggests a higher chance of node \( j \) being matched, and a value closer to 0 suggests a lower chance.

We compute the cost matrix \( C \) between the nodes of \( \mathcal{G}_S \) and \( \mathcal{G}_T \) by applying Sinkhorn normalization \citep{sinkhorn1964relationship} to the cross-graph node-to-node affinity matrix \( A \). This normalization transforms \( A \) into a doubly stochastic matrix \( S \in [0, 1]^{m \times n} \) such that \( S\mathbf{1}_n = \mathbf{1}_m \) and \( S^{\intercal}\mathbf{1}_m \leq \mathbf{1}_n \) \citep{knight2008sinkhorn}. The cost matrix \( C \in [0, 1]^{m \times n} \) is then derived from \( S \), with each element defined as \( C_{ij} = 1 - S_{ij} \) for all nodes \( i \) in $\mathcal{G}_S$ (where \( 1 \leq i \leq m \)) and \( j \)  in $\mathcal{G}_T$ (where \( 1 \leq j \leq n \)).

\paragraph{Partial Matching Loss}
We have two primary learning objectives during training: (1) learning the matching cost matrix 
$C$, and (2) learning the matching biases $\alpha$ and $\beta$ of nodes in the source and target graphs. Below, we propose a new loss function that integrates these two learning objectives.

Let \( \mathbb{I}(\cdot) \) denote the indicator function. Using the feasibility threshold property from \cref{lem:1}, we define the matching attention matrix \( Z \) as:
\[
Z_{ij} = \mathbb{I}\left(M_{ij} = 1 \text{ or } C_{ij} \leq \rho(\alpha_i + \beta_j)\right).
\]

Based on \( Z \) and the ground truth matching matrix \( M \in \mathcal{M} \), we propose the loss term \( L_{cost} \) to guide the learning of the matching cost matrix \( C \):
\[
L_{cost} = - \sum_{i,j} Z_{ij} \left[M_{ij} \log(1 - C_{ij}) + (1 - M_{ij}) \log(C_{ij})\right].
\]
A node pair is included in the loss \( L_{cost} \) under two conditions: (1) The ground truth indicates the pair should be matched (\( M_{ij} = 1 \)), or (2) The pair should not be matched (\( M_{ij} = 0 \)) but the matching cost is below the feasibility threshold (\( C_{ij} \leq \rho(\alpha_i + \beta_j) \)). According to \cref{lem:1}, pairs exceeding the threshold are infeasible and thus are not penalized in the loss, as they cannot produce false positives.

To guide the learning of the matching biases \( \alpha \) and \( \beta \), we propose the loss term \( L_{bias} \):
\[
L_{bias} =  \sum_{i=1}^{m} (M \mathbf{1}_n)_i \left[ 1 - \alpha_i \right]^2 +  \sum_{j=1}^{n} (M^\intercal \mathbf{1}_m)_j\left[ 1 - \beta_j \right]^2.
\]
 The goal of \( L_{bias} \) is to increase the matching bias of a node that should be matched. The partial matching loss is defined as \( L = L_{cost} + \lambda L_{bias} \), where \( \lambda \in (0,1]\) is the regularization parameter.

\section{Experiments}
\label{sec:experiments}

In this section,
we evaluate our proposed partial graph matching approach to empirically verify its performance and robustness, addressing the following aspects: (1) We first evaluate the efficacy of our approach through experiments on image keypoint matching datasets and Protein-Protein Interaction (PPI) networks under varying noise levels; (2) We analyze the effects of the matching biases \( \alpha \) and \( \beta \) on overall performance by considering two model variants: OPGM, with fixed equal matching biases (\( \alpha = \mathbf{1}_m, \beta = \mathbf{1}_n \)), and OPGM-rs, with learnable matching biases. Note that when training OPGM,  $L = L_{cost}$, as we do not need to learn matching biases. (3) We further conduct a sensitivity analysis on the unbalancedness parameter \( \rho \) and regularization parameter \( \lambda \) to understand their impact on partial graph matching performance;  (4) We analyze the runtime efficiency of our approach compared to other baselines. Our code is available at GitHub: \url{https://github.com/Gathika94/OPGMrs.git}.

\subsection{Experimental Setup}
\paragraph{Image Keypoint Matching} 
In this task, we focus on image keypoint matching, which aims to find corresponding annotated keypoints between two given images. We use three image keypoint matching datasets that inherently contain outliers: 1) the \emph{Pascal VOC Keypoint with Berkeley annotations} \citep{everingham2010pascal}, which includes keypoint-annotated images from 20 classes; 2) \emph{SPair-71k} \citep{min2019spair}, featuring 70,958 high-quality image pairs from Pascal VOC 2012 and Pascal 3D+, covering 18 classes; and 3) \emph{IMC PT SparseGM} \citep{wang2023deep}, a recently proposed dataset specifically for partial graph matching.

\smallskip
\noindent\emph{Experimental Setting.~} We followed the same experimental setting and evaluation mechanism as described in ~\citep{wang2023deep} to conduct experiments on the image keypoint matching datasets. Within this experimental setup, a pre-trained VGG16 model~\citep{simonyan2014very} is used to extract visual features of annotated image keypoints. Graphs are created using the extracted keypoint features following the same protocol as in~\citep{wang2023deep}. 
We use the neural network proposed in  GCAN \citep{jiang2022graph} as $NN_{\theta}$ to learn node embeddings and compute the cross-graph node-to-node affinity matrix. 

Consistent with \citep{wang2023deep}, we report the matching F1-score as the evaluation metric. When reporting the results, we run 5 random starts and report the 95\% confidence interval as the error bars, similarly to ~\citep{wang2023deep}. We considered the following baselines:  NGM-v2 ~\citep{wang2020learning}, GCAN~\citep{jiang2022graph}, AFAT-U~\citep{jiang2022graph}, and AFAT-I \citep{wang2023deep}. We consider the implementation and results reported for baselines in \citep{wang2023deep}\footnote{https://github.com/Thinklab-SJTU/ThinkMatch} for evaluations. 
 In ~\cref{tab:spair71K}, and ~\cref{fig:combined_performance_and_inference}(left), GM Network denotes the neural network architecture that each of the given baselines used to learn node embeddings and to obtain the cross-graph node-to-node affinity matrix ($A$). PMH indicates the technique that is used to distinguish matching vs non-matching nodes and to obtain the final partial matching.

\begin{table*}[t!]
    \centering
     \resizebox{\textwidth}{!}
    {\renewcommand{\arraystretch}{1.2}
    \begin{tabular}{l l |c c c c c c c c c c c c c c c c c c| c } 
    \toprule
 GM Network & PMH & aero &	bike & bird	& boat	& bottle & bus & car & cat & chair & cow  & dog & horse & mbike & person & plant & sheep & train & tv & mean \\ 
    \toprule
NGM-v2 & dummy & 47.7   & 41.6   & 62.1   & 30.3   & 59.0   & 49.7   & 27.4   & 68.3   & 33.9   & 62.4   & 57.3   & 46.7   & 46.4   & 42.7   & 78.7   & 43.5   & 80.5   & 89.5	& 53.8 $\pm$ 0.4\\
NGM-v2 & AFAT-U & 50.3 &  43.5 &  63.8 &    \textbf{\textcolor{blue(pigment)}{32.4}} &  59.0 & 
 60.1 &  \textbf{\textcolor{blue(pigment)}{39.7}} &  \textbf{\textcolor{blue(pigment)}{68.6}} &  36.1 &  63.6 &  56.5 &  46.3 & 
   \textbf{\textcolor{blue(pigment)}{51.4}} &  43.3 &  77.0 &  \textbf{51.2} &  81.1 &  89.4	& 56.3 $\pm$ 0.4 \\ 
NGM-v2 & AFAT-I & 50.4 & 43.6 & 63.9 & 32.1 & 61.2 & 58.5 & 38.0 &   68.4 & 35.7 & 62.7 & 56.4 & 47.7 & \textbf{51.9} &   \textbf{\textcolor{blue(pigment)}{44.3}} & 78.5 &   \textbf{\textcolor{blue(pigment)}{50.7}} & 79.2 & 91.2 & 56.4 $\pm$ 0.4  \\

GCAN & ILP & 49.0 & 41.3 & 64.0 & 30.3 & 57.3 & 55.0 & 37.4 & 64.8 & 36.6 & 63.0 &   58.0 & 44.4 & 46.4 & 42.6 & 68.4 & 42.3 &   83.2 &   \textbf{\textcolor{blue(pigment)}{91.9}} & 54.2$\pm$ 0.3\\ 
GCAN & AFAT-U & 46.7 & 43.3 & 65.8 & \textbf{33.3} & \textbf{\textcolor{blue(pigment)}{61.5}} & 54.9 & 35.2 &   68.4 & 37.7 & 59.9 & 56.0 & 47.6 & 47.2 & 43.5 &   \textbf{\textcolor{blue(pigment)}{80.3}} & 47.7 & \textbf{83.8} & 89.0 & 55.7$\pm$ 0.4 \\ 
GCAN & AFAT-I & 46.8 &   \textbf{44.3} & 65.9 &   \textbf{\textcolor{blue(pigment)}{32.4}} & \textbf{\textcolor{blue(pigment)}{61.5}} & 53.8 & 33.7 &   68.4 &   \textbf{\textcolor{blue(pigment)}{38.1}} & 60.1 & 56.3 &   \textbf{\textcolor{blue(pigment)}{47.9}} & 48.3 &43.8 & \textbf{81.2} & 48.4 & 82.9 & 88.0 & 55.7 $\pm$ 0.4\\ 
    \hline
GCAN & OPGM &   \textbf{\textcolor{blue(pigment)}{51.9}}	 & \textbf{\textcolor{blue(pigment)}{43.8}}	 & \textbf{\textcolor{blue(pigment)}{66.6}}	 & 28.9 &	60.9	 & \textbf{\textcolor{blue(pigment)}{60.6}}	 & 37.8	 & 67.8	 & 37.7	 &   \textbf{\textcolor{blue(pigment)}{64.3}}	 & \textbf{\textcolor{blue(pigment)}{58.9}}	 & 47.6	 & 47.8	 & 43.1	 & 77.3	 & 49.5	 & 82.1  &	90.9 &   \textbf{\textcolor{blue(pigment)}{56.5}} $\pm$   \textbf{\textcolor{blue(pigment)}{0.4}} \\
GCAN & OPGM-rs & \textbf{53.0} & 43.5 & \textbf{66.7} & 32.1 & \textbf{61.7} & \textbf{61.4} & \textbf{40.7} & \textbf{68.8} & \textbf{38.5} & \textbf{65.8} & \textbf{59.5} & \textbf{51.1} & 47.2 & \textbf{46.2} & 78.6 & 48.7 & \textbf{\textcolor{blue(pigment)}{83.3}} & \textbf{92.3} & \textbf{57.7} $\pm$ \textbf{0.2} \\
    \bottomrule
    \end{tabular}
    }
\caption{Performance (matching F1-score) on the dataset SPair-71K. The best results are colored in \textbf{black} and the second best are in \textcolor{blue(pigment)}{\textbf{blue}}.}
\label{tab:spair71K}
\vspace{-0.2cm}
\hspace*{0cm}\begin{minipage}{0.62\textwidth} 
    \centering
    \begin{adjustbox}{width=\textwidth}
    \renewcommand{\arraystretch}{1.3} 
    \begin{tabular}{l l | c c c |c | c c c |c}
    \toprule
    \multicolumn{2}{c|}{Dataset Name} & \multicolumn{4}{c|}{IMCPT 50} & \multicolumn{4}{c}{IMCPT 100} \\
    \cline{1-2} \cline{3-10}
GM Network & PMH & reichstag & sacre & st peters & mean & reichstag & sacre & st peters & mean \\
    \hline
   
    NGM-v2 & dummy & 88.5 & 56.1 & 63.0 & 69.2±0.5 & 80.0 & 57.0 & 71.3 & 69.5±0.3 \\
    NGM-v2 & AFAT-U & 90.5 & 58.7 & 66.9 & 72.0±0.3 & 81.7 & 57.0 & 72.2 & 70.3±0.2 \\
    NGM-v2 & AFAT-I & \textbf{92.3} & 58.7 & 66.7 & \textbf{\textcolor{blue(pigment)}{72.8}}±\textbf{\textcolor{blue(pigment)}{0.4}} & 82.0 & 57.0 & 71.4 & 70.1±0.3 \\
  
    GCAN & ILP & 87.2 & 55.1 & 63.0 & 68.4±0.5 & 80.4 & 55.7 & 72.8 & 69.6±0.4 \\
    GCAN & AFAT-U & 86.9 & \textbf{\textcolor{blue(pigment)}{59.4}} & 67.1 & 71.1±0.4 & \textbf{\textcolor{blue(pigment)}{82.6}} & 58.2 & 73.8 & 71.5±0.2 \\
    GCAN & AFAT-I & 91.0 & \textbf{60.3} & \textbf{\textcolor{blue(pigment)}{67.3}} & \textbf{72.9}±\textbf{0.6} & \textbf{82.7} & 57.8 & 72.4 & 70.9±0.4 \\
    \hline
    GCAN & OPGM & 91.2 & 57.0 & \textbf{68.1} & 72.1±0.2 & 81.7 & \textbf{\textcolor{blue(pigment)}{59.1}} & \textbf{76.1} & \textbf{\textcolor{blue(pigment)}{72.3}}±\textbf{\textcolor{blue(pigment)}{0.4}} \\
    GCAN & OPGM-rs & \textbf{\textcolor{blue(pigment)}{91.9}} & 59.2 & 67.0 & 72.7±0.1 & 82.3 & \textbf{60.5} & \textbf{\textcolor{blue(pigment)}{75.7}} & \textbf{72.8±0.1} \\
    \bottomrule
    \end{tabular}
    \end{adjustbox}
\end{minipage}%
\hfill
\begin{minipage}{0.38\textwidth}
\vspace{0.4cm}
\hspace{0.5cm}\begin{tikzpicture}
    \begin{axis}[
        width=0.8\textwidth,
        title={Inference time},
        title style={font=\fontsize{8}{5}\selectfont, yshift=-1ex},
        ylabel={(ms)},
        ylabel style={font=\fontsize{8}{5}\selectfont},
         ylabel style={font=\small, yshift=-5pt}, 
        ylabel style={font=\small},
        symbolic x coords={OPGM-rs, AFAT-I, AFAT-U, GCAN},
        xtick=data,
        xticklabels={OPGM-rs, AFAT-I, AFAT-U, GCAN},
        tick label style={font=\footnotesize},
        xticklabel style={rotate=40, anchor=north east, xshift=10pt}, 
        nodes near coords,
        nodes near coords style={font=\tiny},
        ybar,
        bar width=10pt,
        ymin=0,
        ymax=675,
        enlarge x limits=0.2,  
        every axis label/.append style={font=\tiny},
        every tick label/.append style={font=\tiny},
        axis line style={draw=none},  
        tick style={draw=none},       
        x axis line style={draw},     
        y axis line style={draw},     
    ]
    \addplot[fill=teal!60] coordinates {(OPGM-rs,117.8) (AFAT-I,367.1) (AFAT-U,371.8) (GCAN,561.1)};
    \end{axis}
\end{tikzpicture}

\end{minipage}

\vspace{-0.4cm}
\caption{(left) Performance (matching F1-score) on the datasets IMCPT 50 and IMCPT 100, where the best results are colored in \textbf{black} and the second best are in \textcolor{blue(pigment)}{\textbf{blue}}; (right) Inference time  of our model OPGM-rs against the state-of-the-art methods AFAT-I, AFAT-U, and GCAN on the dataset IMCPT 100.}
\label{fig:combined_performance_and_inference}
\end{table*}

\paragraph{PPI Network Matching}

In this task, we focus on Protein Protein Interaction network matching under varying noise levels. PPI network matching dataset is a standard graph matching benchmark that can be used to evaluate performance of a graph matching model under various noise levels ~\citep{liu2021stochastic,ratnayaka2023contrastive}. It is a protein protein interaction (PPI) network of yeasts, consisting of 1004 proteins and
4920 high-confidence interactions among those proteins. The PPI network matching problem is to match this network with  its noisy versions, which contain 5\%, 10\%, 15\%,20\%,25\% additional interactions (low-confidence interactions), respectively.

\smallskip
\noindent\emph{Experimental Setting.~} 
We followed the same experimental setting used in the baselines we consider, SIGMA~\citep{liu2021stochastic} and StableGM~\citep{ratnayaka2023contrastive}. Node correctness (the percentage of nodes that have the same
matching as the ground truth) is used as the evaluation metric, similarly to the previous works ~\citep{liu2021stochastic,ratnayaka2023contrastive}.
We use the same neural network architecture that has been used in ~\citep{liu2021stochastic,ratnayaka2023contrastive} as  $\textsc{nn}_{\theta}$, where the input feature of each node is calculated based on the degree of the node and a 5-layer Graph Isomorphism Network (GIN)~\citep{xu2018powerful} was used. The cross-graph node-to-node affinity matrix is calculated based on the cosine similarities of node embeddings. 

It should be noted that PPI network matching is a graph matching task, aimed at finding a bijective mapping between two graphs with an equal number of nodes. From ~\cref{eq:PGM} and ~\cref{lem:1}, a bijective mapping can be achieved from our method for higher the values of $\rho$. Therefore, for PPI network matching, we set the hyperparameter $\rho=10^{11}$.

\begin{figure}[htbp]
\begin{minipage}{0.33\textwidth}
    \begin{tikzpicture}
        \begin{axis}[
            width=\textwidth,
            height=.85\textwidth,
            xlabel={Noise (\%)},
            xlabel style={font=\fontsize{8}{10}\selectfont},  
            ylabel={Node Correctness (\%)},
            ylabel style={font=\fontsize{8}{5}\selectfont},
            legend style={
                at={(0.7,0.8)},
                anchor=south,
                font=\fontsize{6}{5}\selectfont,
                legend cell align=left,
                draw=none,
                fill=none,
                legend columns=2,
                /tikz/every legend/.append style={
                    draw=none,
                }
            },
            legend image post style={scale=0.5},
            xtick={1,2,3,4,5},
            xticklabels={5, 10, 15, 20, 25},
            x tick label style={font=\tiny},
            enlarge x limits={abs=0.05},
            xmin=1, xmax=5.3,
            ymin=40, ymax=95,
            grid=both,
            grid style={line width=.1pt, draw=gray!10},
            every axis label/.append style={font=\small},
            every axis title/.append style={font=\small},
            tick label style={font=\tiny},
            axis lines=left,
            clip=false,
             ytick={40,50,60,70,80},
             yticklabels={40,50,60,70,80} 
        ]
        \addplot[
            color=blue,
            line width=0.6pt,
            mark=o,
            mark size=1.5pt,
            mark options={fill=blue}
        ]
        coordinates {
            (1,84.7) (2,68.8) (3,57.4) (4,46.7) (5,41.4)
        };
        \addlegendentry{SIGMA}
        \addplot[
            color=red,
            line width=0.6pt,
            mark=o,
            mark size=1.5pt,
            mark options={fill=red}
        ]
        coordinates {
            (1,86.1) (2,75.6) (3,67.9) (4,63.2) (5,57.0)
        };
        \addlegendentry{StableGM}
        
        \addplot[
            color=orange,
            line width=0.6pt,
            mark=o,
            mark size=1.5pt,
            mark options={fill=orange}
        ]
        coordinates {
            (1,88.3) (2,79.8) (3,71.7) (4,66.2) (5,57.7)
        };
        \addlegendentry{OPGM-rs}
    \end{axis}
    \end{tikzpicture}
\end{minipage}
\hfill
\begin{minipage}{0.32\textwidth}
    \begin{tikzpicture}
        \begin{axis}[
            width=\textwidth,
            height=.85\textwidth,
            xlabel={$\rho$},
            ylabel={F1-Score (\%)},
            ylabel style={font=\fontsize{8}{5}\selectfont},
            legend style={
                at={(0.5,1.0)},
                anchor=south,
                font=\fontsize{6}{5}\selectfont,
                legend cell align=left,
                draw=none,
                fill=none,
                legend columns=-1,
                /tikz/every legend/.append style={
                    draw=none,
                    column sep=1ex,
                }
            },
            legend image post style={scale=0.5},
            xtick={0.1, 0.2, 0.3, 0.4, 0.5},
            xticklabels={0.1, 0.2, 0.3, 0.4, 0.5},
            x tick label style={font=\tiny},
            enlarge x limits={abs=0.05},
            xmax=0.55,
            ymin=40,
            ymax=83,
            grid=both,
            grid style={line width=.1pt, draw=gray!10},
            every axis label/.append style={font=\small},
            every axis title/.append style={font=\tiny},
            tick label style={font=\tiny},
            axis lines=left,
            clip=false,
            ]
            \addplot[
                color=blue,
                line width=0.6pt,
                error bars/.cd, 
                y dir=both, 
                y explicit
            ]
            coordinates {
                (0.1,45.3) +- (0,0.4)
                (0.2,52.0) +- (0,0.4)
                (0.3,56.5) +- (0,0.2)
                (0.4,57.7) +- (0,0.2)
                (0.5,56.1) +- (0,0.4)
            };
            \addlegendentry{SPair-71K}
            
            \addplot[
                color=red,
                line width=0.6pt,
                error bars/.cd, 
                y dir=both, 
                y explicit
            ]
            coordinates {
                (0.1,51.2) +- (0,0.4)
                (0.2,63.5) +- (0,0.3)
                (0.3,71.6) +- (0,0.2)
                (0.4,72.8) +- (0,0.1)
                (0.5,70.9) +- (0,0.3)
            };
            \addlegendentry{IMCPT 100}
        \end{axis}
    \end{tikzpicture}
\end{minipage}
\hfill
\begin{minipage}{0.33\textwidth}
    \begin{tikzpicture}
        \begin{axis}[
            width=\textwidth,
            height=.85\textwidth,
            xlabel={$\lambda$},
            ylabel={F1-Score (\%)},
            ylabel style={font=\fontsize{8}{5}\selectfont},
            legend style={
                at={(0.5,1.0)},
                anchor=south,
                font=\fontsize{6}{5}\selectfont,
                legend cell align=left,
                draw=none,
                fill=none,
                legend columns=-1,
                /tikz/every legend/.append style={
                    draw=none,
                    column sep=1ex,
                }
            },
            legend image post style={scale=0.5},
            xmin=0, xmax=1.05,
            ymin=40, ymax=83,
            grid=both,
            grid style={line width=.1pt, draw=gray!10},
            every axis label/.append style={font=\small},
            every axis title/.append style={font=\tiny},
            tick label style={font=\tiny},
            axis lines=left,
            clip=false,
            xtick={0.0,0.2,0.4,0.6,0.8,1.0},
            xticklabels={0,0.2,0.4,0.6,0.8,1},
        ]
            \addplot[
                purple,
                line width=0.6pt,
                error bars/.cd, 
                y dir=both, 
                y explicit
            ]
            coordinates {
                (0.01,54.0) +- (0,0.4)
                (0.2,57.1) +- (0,0.2)
                (0.4,57.5) +- (0,0.2)
                (0.5,57.7) +- (0,0.2)
                (0.6,57.4) +- (0,0.3)
                (0.8,57.5) +- (0,0.3)
                (1,57.4) +- (0,0.2)
            };
            \addlegendentry{SPair-71K}
            
            \addplot[
                magenta,
                line width=0.6pt,
                error bars/.cd, 
                y dir=both, 
                y explicit
            ]
            coordinates {
                
                (0.01,71.3) +- (0,0.6)
                (0.2,72.5) +- (0,0.2)
                (0.4,72.8) +- (0,0.1)
                (0.6,72.5) +- (0,0.2)
                (0.8,72.6) +- (0,0.2)
                (1,72.3) +- (0,0.2)
            };
            \addlegendentry{IMCPT 100}
        \end{axis}
    \end{tikzpicture}
\end{minipage}\vspace{-0cm}
\caption{(left) Performance (node correctness) of our model OPGM-rs against SIGMA and StableGM on PPI Network Matching with varying noise levels; (middle) Sensitivity analysis of our model OPGM-rs on the unbalancedness parameter $\rho$; (right) Sensitivity analysis of our model OPGM-rs on the regularization parameter $\lambda$. Performance in sensitivity analysis is measured by matching F1-score.}
\label{fig:combined_analysis}
\end{figure}

\subsection{Results and Discussion}

\paragraph{Ext-1 Architecture efficacy}
When examining the results for the Spair-71K dataset in~\cref{tab:spair71K}, both OPGM and OPGM-rs outperform other baselines in terms of mean F1-score. Moreover, OPGM-rs or OPGM outperforms all other baselines in 9 out of 18 classes. It is important to note that the Spair-71K dataset offers several advantages over the Pascal VOC Keypoint dataset, including higher image quality, richer keypoint annotations, and the removal of ambiguous annotations~\citep{rolinek2020deep}.  Therefore, the superior performances of our models on Spair-71K dataset demonstrate the effectiveness of our approach for visual graph matching. 

\cref{fig:combined_performance_and_inference} (left) shows the results for the IMCPT 50 and IMCPT 100 datasets. Our models outperform the baselines in terms of mean F1-score on the IMCPT 100 dataset and in 2 out of 3 classes. For the IMCPT 50 dataset, our method performs comparably to the best baselines. Notably, the IMCPT 100 dataset, which contains the highest number of nodes among visual graph matching datasets~\citep{wang2023deep}, further highlights the efficacy of our approach for partial graph matching. We discuss the evaluations on the Pascal VOC Keypoint dataset in the appendix, where we observe that poor and ambiguous annotations in data can adversely affect the performances of our proposed models.

 As shown in \cref{fig:combined_analysis}~(left), our models OPGM-rs consistently outperforms other baselines across all noise levels for PPI network matching (see the appendix for numerical results). However, at higher noise levels, the performance gap between our model and StableGM narrows. 

\paragraph{Ext-2 Impact of matching biases}
From the results given in ~\cref{tab:spair71K,fig:combined_performance_and_inference} and ~\cref{fig:combined_analysis}~(left), it is clear that OPGM-rs consistently outperforms OPGM in most cases, highlighting the importance of learning the matching biases of nodes in partial graph matching.

\paragraph{Ext-3 Sensitivity analysis}We evaluated the impact of the unbalancedness parameter \(\rho\) and the regularization parameter \(\lambda\) on partial graph matching performance using Spair-71K and IMCPT 100 datasets.
For \(\rho\), we varied its values while keeping other hyperparameters fixed. As shown in \cref{fig:combined_analysis}~(middle), mean F1-scores highlight \(\rho\)'s critical role: smaller values restrict valid matches, while larger values may allow incorrect matches, both reducing performance. For \(\lambda\), starting from \(0.01\), we tested \(\lambda \in \{0.2, 0.4, 0.6, 0.8, 1\}\) with fixed hyperparameters. As shown in \cref{fig:combined_analysis}~(right), the mean F1-scores show slight variation, indicating that the impact of \(\lambda\) is not significant.

\paragraph{Ext-4 Efficiency analysis}We evaluated the average runtime for processing a pair of graphs (i.e., keypoint-annotated images) during the inference phase for OPGM-rs, GCAN, AFAT-I, and AFAT-U using the IMCPT 100 dataset, the largest image keypoint matching dataset~\citep{wang2023deep}. Matching was performed on 3,000 keypoint-annotated pairs. 
As shown in~\cref{fig:combined_performance_and_inference}~(right), OPGM-rs demonstrates significantly lower inference time compared to the other models. It is important to note that all models use the same neural network architecture from ~\citep{jiang2022graph} (corresponding to $\textsc{nn}_{\theta}$) to compute the cross-graph node-to-node affinity matrix $A$ and the Sinkhorn algorithm to derive the doubly stochastic affinity matrix $S$, differing only in their partial matching techniques.

\section{Conclusion and Limitations}
\label{sec:conclusion}

In this work, we proposed a new problem formulation for partial graph matching based on optimal partial transport. Our approach can dynamically determine which nodes would be matched or left unmatched during the optimization process. We demonstrate how our problem formulation enabled solving the partial graph matching problem within cubic worst case time complexity. We then showed how our proposed solution for partial graph matching can be effectively integrated to a learning setting. Evaluations on various partial graph matching benchmarks demonstrate that our method outperform the baselines in most of the benchmarks. Beyond these specific tasks discussed in this work, the optimization problem we proposed and the solution mechanism we developed can be adapted to any application requiring optimal partial assignment between two sets.

While our proposed method demonstrated strong performance on datasets with reliable annotations, its effectiveness diminished when faced with unreliable or ambiguous annotations, or higher noise levels in the data. Enhancing the robustness of our method under these challenging conditions is an important direction for future work.

\section*{Acknowledgements}
This research was supported partially by the Australian Government through the Australian Research Council's Discovery Projects funding scheme (project DP210102273).  

\bibliography{references}
\bibliographystyle{iclr2025_conference}

\appendix
\clearpage
\section{Appendix}

\subsection{Equivalence between Optimal Transport and Optimal Partial Transport}
\label{subsec:POTeqOT}

Inspired by ~\citep{bai2023sliced,caffarelli2010free}, we show that the optimal partial transport problem proposed in ~\cref{eq:UGM} has an equivalance with an optimal (balanced) transport problem. 
In the following we will write $\| \mu \|_1 = \sum_{i=1}^m | \mu_i |$ or $\| \pi \|_1 = \sum_{i,j = 1}^{m,n} | \pi_{ij} |$ to denote the 1-norm or sum of elements of probability vectors and matrices noting that all $\mu_i, \pi_{ij} > 0$.

Let $K$ be a constant satisfying 
$K \geq \| \mu \|_1 + \| \nu \|_1$,  
and we define extended vectors $\hat{\nu}, \hat{\mu} \in \mathbb{R}^{m+n}$ as 
\[
    \hat{\mu}_i= 
    \begin{cases}
    \mu_i & \text{if } i \leq m, \\
    \frac{1}{n}(K-\|\mu\|_1) & \text{if } m < i \le m + n
\end{cases}
\quad\text{and}\quad
\hat{\nu}_j= 
    \begin{cases}
    \nu_j & \text{if } j \leq n, \\
    \frac{1}{m}(K-\|\nu\|_1) & \text{if } n < j \le m + n
\end{cases}
\]
and note that $\| \hat{\mu} \|_1 = \| \hat{\nu} \|_1 = K$.
We define $\widehat{C} \in \mathbb{R}^{(m+n) \times (m+n)}$ such that
 \begin{equation} \label{eq:C_hat_app}
    \widehat{C}_{ij}= 
\begin{cases}
    C_{ij} - \rho(\alpha_{i}+\beta_{j})& \text{if } i \leq m \text{ and }j \leq n, \\
    0 & \text{if }  \text{otherwise.} 
\end{cases}
\end{equation}
We consider the following balanced optimal transport problem 
between the extended vectors $\hat{\mu}$ to $\hat{\nu}$ that optimizes over the set of admissible couplings $    \Pi(\hat{\mu},\hat{\nu})= \{ \pi \in \mathbb{R}_{\geq 0}^{(m+n)\times (m+n)} | \pi \mathbf{1}_{m+n}=\hat{\mu}, \pi^{\intercal  }\mathbf{1}_{m+n} =\hat{\nu} \}
$.
\begin{equation}
\label{eq:OTeqv}    \OT(\hat{\mu},\hat{\nu}) = \min_{{\pi} \in \Pi(\hat{\mu}, \hat{\nu})}\;\langle\,\pi,\widehat{C} \rangle.
\end{equation}

We claim that there exists an equivalence between this OT problem and ~\cref{eq:UGM}.

\begin{restatable}[Equivalent cost of couplings]{proposition}{PropEquiv} 
\label{lem:PropEquiv}
For any $\hat \pi, \hat \pi^\prime \in \Pi(\hat \mu, \hat \nu)$ with 
$\hat{\pi}[1:m][1:n] = \hat{\pi}^\prime[1:m][1:n] = \pi$
\[
\langle \widehat C,\hat \pi \rangle_F   
= \langle \widehat C, \hat \pi^\prime \rangle_F
\]
and
\[
 \TC(\pi; C, \alpha, \beta)
    = \langle \widehat{C}, \hat{\pi} \rangle_F + \rho \left( \langle \alpha, \mu \rangle + \langle \beta, \nu \rangle \right).
\]
\end{restatable}

\begin{proof}

Based on the definition of $\widehat{C}$, we can derive that for any $\pi \in \Pi(\hat{\mu},\hat{\nu})$,

\begin{align*}
 \langle \widehat{C}, \pi \rangle_F = \sum_{i=1}^{m}\sum_{j=1}^{n} \widehat{C}_{ij}\pi_{ij}
\end{align*}

Together with the above result and the fact that $\hat{\pi}[1:m][1:n] = \hat{\pi}^\prime[1:m][1:n] = \pi$, we can derive

\begin{align}
\label{eq:totalcostequivalance}
 \langle \widehat{C}, \hat{\pi} \rangle_F = \sum_{i=1}^{m}\sum_{j=1}^{n} \widehat{C}_{ij}\hat{\pi}_{ij}=\sum_{i=1}^{m}\sum_{j=1}^{n} \widehat{C}_{ij}\pi_{ij} =\sum_{i=1}^{m}\sum_{j=1}^{n} \widehat{C}_{ij}\hat{\pi}^\prime_{ij}= \langle \widehat{C}, \hat{\pi}^\prime \rangle_F
\end{align}

Thus,  $ \langle \widehat{C}, \hat{\pi} \rangle_F = \langle \widehat{C}, \hat{\pi}^\prime \rangle_F$.

From ~\cref{eq:UGM-tc}, we know that,

\begin{align*}
    \TC(\pi; C, \alpha, \beta) 
    &=  \langle \pi, C \rangle_F + \rho\left( \langle \alpha, \mu - \pi_1 \rangle 
    + \langle \beta, \nu - \pi_2 \rangle \right)\\ 
   &= \sum_{i=1}^{m}\sum_{j=1}^{n} \left( C_{ij} - \rho (\alpha_i + \beta_j) \right)\pi_{ij} + \rho \left( \langle \alpha, \mu \rangle + \langle \beta, \nu \rangle \right) \\
    &= \langle \widehat{C}, \hat{\pi} \rangle_F + \rho \left( \langle \alpha, \mu \rangle + \langle \beta, \nu \rangle \right).
\end{align*}

where in the last step we used ~\cref{eq:totalcostequivalance}.
Thus, the proof is complete.
\end{proof}
\begin{restatable}[Equivalence of minimizers]{proposition}{PropMinim} 
For any $\hat\pi \in \Pi(\hat \mu, \hat \nu)$ that minimizes \eqref{eq:OTeqv} 
then the top left corner  $\pi = \hat{\pi}[1:m][1:n]$ satisfies $\pi \in \Pi_{\le}(\mu, \nu)$, and $\pi$ furthermore minimizes \eqref{eq:UGM}.
\end{restatable}

\begin{proof}
    For any $\hat\pi \in \Pi(\hat \mu, \hat \nu)$, let $\hat{\pi}[1:m][1:n]=\pi$.

From the definitions of $~\hat{\mu}$ and $~\hat{\nu}$, we know that $~\hat{\mu}[1:m]=\mu$ and $~\hat{\nu}[1:n]=\nu$. Therefore, $\pi_1 \leq \mu$ and $\pi_2 \leq \nu$. Thus, $\pi \in \Pi_{\le}(\mu, \nu)$.

Conversely, for any $\pi \in \Pi_{\leq}(\mu, \nu)$, we can construct 
$\hat{\pi} \in \Pi(\hat{\mu}, \hat{\nu})$ with $\hat{\pi}[1:m][1:n] = \pi$. 
After placing $\pi$ in the top-left corner of $\hat{\pi}$, we fill the 
remaining blocks to satisfy the marginal constraints. The block 
$\hat{\pi}[1:m][n{+}1:m{+}n]$ has $\hat{\pi}_{ij} = \frac{\mu_i - (\pi_1)_i}{m}$, 
assigning surplus mass at each source uniformly across the extended columns. 
The block $\hat{\pi}[m{+}1:m{+}n][1:n]$ has 
$\hat{\pi}_{ij} = \frac{\nu_j - (\pi_2)_j}{n}$, assigning deficit mass at 
each target uniformly across the extended rows. The block 
$\hat{\pi}[m{+}1:m{+}n][n{+}1:m{+}n]$ has 
$\hat{\pi}_{ij} = \frac{K - \|\mu\|_1 - \|\nu\|_1 + \|\pi\|_1}{mn}$. 
All entries are non-negative since $\pi \in \Pi_{\leq}(\mu, \nu)$ and 
$K \geq \|\mu\|_1 + \|\nu\|_1$, and the mass is assigned such that the 
marginal constraints $\hat{\pi}\mathbf{1}_{m+n} = \hat{\mu}$ and 
$\hat{\pi}^\top\mathbf{1}_{m+n} = \hat{\nu}$ are satisfied.

From ~\cref{lem:PropEquiv}, we know that,

\begin{equation*}
\TC(\pi; C, \alpha, \beta) = \langle \widehat{C}, \hat{\pi} \rangle_F + \rho \left( \langle \alpha, \mu \rangle + \langle \beta, \nu \rangle \right).
\end{equation*}

We also know that for any $\hat\pi \in \Pi(\hat \mu, \hat \nu)$, $\exists$ $\pi \in \Pi_{\le}(\mu, \nu)$ s.t. $\pi = \hat{\pi}[1:m][1:n]$. Moreover, $\rho \left( \langle \alpha, \mu \rangle + \langle \beta, \nu \rangle \right)$ is a  constant. Therefore, $\TC(\pi; C, \alpha, \beta)$ is minimized when $\langle \widehat{C}, \hat \pi \rangle_F$ is minimized. Consequently, for any $\hat\pi \in \Pi(\hat \mu, \hat \nu)$ that minimizes \cref{eq:OTeqv} 
then the top left corner  $\pi = \hat{\pi}[1:m][1:n]$ satisfies $\pi \in \Pi_{\le}(\mu, \nu)$, and $\pi$ furthermore minimizes \cref{eq:UGM}.  

\end{proof}

\subsection{Proofs}

\lemmaopt* 

\begin{proof}

When total variation is assumed as the divergence function $D(.|.)$, ~\cref{eq:UGM-tc} can be written as 
\[
\TC(\pi; C, \alpha, \beta) \coloneqq \langle \pi, C \rangle_F + \rho\left( \langle \alpha, \mu - \pi_1 \rangle 
    + \langle \beta, \nu - \pi_2 \rangle \right).\]

For any given source sample $p \in V_S$ and any given target sample $q \in V_T$, we can write $\TC(\pi; C, \alpha, \beta)$ by separating the terms related to  masses of $p \in V_S$ and $q  \in V_T$ (See \cref{eq:TC} as well). 

\begin{align*}
    \TC(\pi; C, \alpha, \beta) 
    = \sum_{\substack{1\le i \le m \\ i \neq p}}\sum_{\substack{1\le j \le n \\ i \neq q}} \pi_{ij}C_{ij}+ \sum_{\substack{1\le j \le n \\ i \neq q}} \pi_{pj}C_{pj}+ \sum_{\substack{1\le i \le m \\ i \neq p}} \pi_{iq}C_{iq} +  \pi_{pq}C_{pq}\\
    + \rho \left(\sum_{\substack{1\le i \le m \\ i \neq p}} \alpha_{i}(\mu_{i}-(\pi_{1})_{i})+ \alpha_{p}\left(\mu_{p}-\sum_{\substack{1\le j \le n \\ i \neq q}} \pi_{pj}-\pi_{pq}\right)\right)
    \\
    +\rho \left( \sum_{\substack{1\le j \le n \\ i \neq q}} \beta_{j}(\nu_{j}-(\pi_{2})_{j}) +  \beta_{q} \left(\nu_{q}-\sum_{\substack{1\le i \le m \\ i \neq p}} \pi_{iq}-\pi_{pq}\right)\right) 
\end{align*}

We consider two plans $\pi, \pi^\prime \in \Pi_{\leq}(\mu, \nu)$ such that $\pi_{pq} \neq \pi^\prime_{pq}$, but otherwise $\pi_{ij} = \pi^\prime_{ij}$ for all other $1 \le i \le m$ ,$1 \le j \le n$. Note that this is possible as they are partial transport plans that do not have strict marginal equality constraints. In this case we find that
\[
\TC(\pi; C, \alpha, \beta) - \TC( \pi^\prime; C, \alpha, \beta)
= (\pi_{pq} - \pi^\prime_{pq}) ( C_{pq} - \rho (\alpha_p + \beta_q))
\]
We note that this implies strict monotonicity of the total transport cost if $C_{pq} - \rho (\alpha_p + \beta_q) > 0$, i.e. that $\TC(\pi; C, \alpha, \beta) > \TC(\pi^\prime; C, \alpha, \beta)$ if $\pi_{pq} > \pi^\prime_{pq}$.

Therefore, if $C_{pq} - \rho (\alpha_p + \beta_q) > 0$ then the optimal strategy is always going to be to set $\pi_{pq} = 0$, ensuring the lowest total cost possible. 
In other words, 
if for any sample $p \in V_S$ and $q \in V_T$, $C_{pq} > \rho(\alpha_{p}+\beta_{q})$, it is always less costly to transport no mass from $p$ to $q$. 
 
Therefore, if $\pi^* \in \Pi_{\le}(\mu, \nu)$ be any optimal solution of ~\cref{eq:UGM}. Then, $\forall i \in V_S,\forall j \in V_T $, if $C_{ij}>\rho(\alpha_i+\beta_j)$ then  $\pi_{ij}^\star = 0$.  

\end{proof}

\lemmaoto*

\begin{proof}

Let us define extended vectors $\hat{\nu} = \mathbf{1}_{m+n}, \hat{\mu} = \mathbf{1}_{m+n} $. Once again we make use of the matrix $\widehat{C}$ as defined in \cref{eq:C_hat_app}.

Next, we consider the following optimal transport problem that optimizes over the set of admissible couplings $    \Pi(\hat{\mu},\hat{\nu})= \{ \pi \in \mathbb{R}_{\geq 0}^{(m+n)\times (m+n)} | \pi \mathbf{1}_{m+n}=\hat{\mu}=\mathbf{1}_{m+n}, \pi^{\intercal  }\mathbf{1}_{m+n} =\hat{\nu}=\mathbf{1}_{m+n} \}
$.

\begin{equation}
\label{eq:OTeqvOpt2}    \OT(\hat{\mu},\hat{\nu}) = \min_{{\pi} \in \Pi(\hat{\mu}, \hat{\nu})}\;\langle\,\pi,\widehat{C} \rangle.
\end{equation}

From the results of ~\cref{subsec:POTeqOT}, we can derive that the optimization problem $\OT(\hat{\mu},\hat{\nu})$ given in ~\cref{eq:OTeqvOpt2} has a clear equivalence to the optimization problem $\OPT_{\rho}(\mu,\nu)$ given in ~\cref{eq:UGM} when $K=m+n$. From ~\cref{subsec:POTeqOT}, we also know that the optimal plan for ~\cref{eq:UGM} can be obtained by restricting an optimal plan $\hat{\pi} \in \argmin_{\pi \in \Pi(\hat{\mu}, \hat{\nu})}\;\langle\,\pi,\widehat{C} \rangle$ to the set $[1:m] \times [1:n]$.

It should be noted that when $\hat{\nu} = \mathbf{1}_{m+n}$, and $\hat{\mu} = \mathbf{1}_{m+n} $, the set $\Pi(\hat{\mu},\hat{\nu})$ corresponds to the set of doubly stochastic matrices. By the Birkhoff-von-Neumann theorem, it is known that the set of doubly stochastic matrices forms a convex polytope and its extremal points are permutation matrices. From the Linear Programming theory, it is known that at least one optimal solution to a linear program over a polytope (here, it is the set of doubly stochastic matrices) lies at an extreme point of the polytope.  Thus, there always exist an optimal plan $\hat{\pi} \in \Pi(\hat{\mu},\hat{\nu})$ which is a permutation matrix. Therefore, the optimal plan $\pi^\star \in \argmin_{\Pi_{\leq}(\mu,\nu)}\TC(\pi; C, \alpha, \beta)$  that is obtained by restricting such $\hat{\pi}\in \argmin_{\pi \in \Pi(\hat{\mu}, \hat{\nu})}\;\langle\,\pi,\widehat{C} \rangle$ (which is a permutation matrix) to its restriction to $[1:m] \times [1:n]$ (i.e., $\pi^{*}=\hat{\pi}[1:m][1:n]$) will be a binary matrix with atmost one 1 per row or column. Therefore, we can always find an optimal plan $\pi^\star \in \argmin_{\Pi_{\leq}(\mu,\nu)}\TC(\pi; C, \alpha, \beta)$ s.t. $\pi^\star \in \mathcal{M}$.
\end{proof}

\theoremeqLAP*
\begin{proof}
We first recall the definition of $\overline C$ given component-wise in \cref{eq:C_bar} by
\[
    \overline{C}_{ij}= 
\begin{cases}
    C_{ij},& \text{if } i \leq m \text{ and }C_{ij} \leq \rho(\alpha_i+\beta_j) \\
    \rho(\alpha_i+\beta_j) ,& \text{if }  i \leq m \text{ and }C_{ij} > \rho(\alpha_i+\beta_j) \\
    \rho(\alpha_*+\beta_j) ,& \text{if }  i >m     
\end{cases}
\]
We use case-by-case definition to decompose the Frobenius inner product. 

We now define index sets of rows and columns which indicate where $\overline \pi$ takes a value of 1 that violates the conditions for setting $\overline C_{ij}$, that is
\begin{align*}
    \mathcal{I} &= \{(i,j) : 1 \le i \le m, 1 \le j \le n, \text{ such that } \overline \pi_{ij} = 1 \text{ and } C_{ij} \le \rho (\alpha_i + \beta_j) \} \\
    \mathcal{J} &= \{ (i,j) : 1 \le i \le m, 1 \le j \le n, \text{ such that } \overline \pi_{ij} = 1 \text{ and } C_{ij} > \rho (\alpha_i + \beta_j) \} \\
    \mathcal{K} &= \{ (i,j) : m < i \le n, 1 \le j \le n, \text{ such that } \overline \pi_{ij} = 1 \}
\end{align*}
and we note that the sets are disjoint within $\{1,\ldots, n\} \times \{1, \ldots, n\}$.
We now decompose the inner product as follows
\begin{align*}
\langle \overline \pi, \overline C \rangle
&= \sum_{(i,j) \in \mathcal{I}} C_{ij} h(\overline \pi)_{ij}
+ \rho \left(\sum_{(i,j) \in \mathcal{J}}  (\alpha_i + \beta_j) \overline \pi_{ij}
+ \sum_{(i,j) \in \mathcal{K}} (\alpha_\star + \beta_j) \overline \pi_{ij} \right)
\end{align*}
where we have immediately used the fact that $\overline \pi_{ij} = h(\overline \pi)_{ij}$ for all $(i,j) \in \mathcal{I}$. 
Note furthermore that $\mathcal{I}$ is exactly the index set where $h(\overline \pi)$ is non-zero, hence that
\[
\sum_{(i,j) \in \mathcal{I}} C_{ij} h(\overline \pi)_{ij} 
= \langle C, h(\overline \pi) \rangle_F.
\]
We recall the notation $h(\overline \pi)_1$ that denotes the first marginal given by $(h(\overline \pi)_1)_i = \sum_{j = 1}^n h(\overline \pi)_{ij}$.
For any fixed row index $i$ where for some $j$ the pair $(i,j) \in \mathcal{J}$, then the the marginal vectors have values $(\overline \pi_1)_i = 1$ whereas $(h(\overline \pi)_1)_i = 0$. On the other hand, if for fixed $i$ if there is no pair $(i,j) \in \mathcal{J}$, then the marginal is given by $(h(\overline \pi)_1)_i = 1$ but it does not contribute to the sum, and we see that we have
\[
\sum_{(i,j) \in \mathcal{J}} \alpha_i \overline \pi_{ij} 
= \sum_{i=1}^m (1 - (h(\overline \pi)_1)_i) \alpha_i.
\]
We have a similar logic for the column sums, with the exception that some of the marginal mass is accounted for in the index set $\mathcal{K}$, so similar to above we have that
\[
\sum_{(i,j) \in \mathcal{J}} \beta_j \overline \pi_{ij} 
+ \sum_{(i,j) \in \mathcal{K}} \beta_j \overline \pi_{ij} 
= \sum_{j=1}^n (1 - (h(\overline \pi)_2)_j) \beta_j.
\]
Finally, as $\overline \pi$ is a complete permutation, we know that each row $m < i \le n$ has some index $j$ for which $\overline \pi_{ij} = 1$ and hence 
\[
\sum_{(i,j) \in \mathcal{K}} \alpha_\star \overline \pi_{ij}
= (n - m) \alpha_\star
\]
All put together this yields
\begin{align*}
\langle \overline \pi, \overline C \rangle
&= \langle C_, h(\overline \pi) \rangle_F
+ \rho \left(\sum_{i=1}^m  \alpha_i (1 - (h(\overline \pi)_1)_i)
+ \sum_{j=1}^n  \beta_j (1 - (h(\overline \pi)_2)_j)
+ (n - m) \alpha_\star \right) \\
&= \langle C_, h(\overline \pi) \rangle_F
+ \rho \left(\langle \alpha, \mathbf{1}_m - h(\overline \pi)_1 \rangle
+ \langle \beta, \mathbf{1}_n - h(\overline \pi)_2 \rangle \right)
+ \rho (n - m) \alpha_\star \\
&= \TCS(h(\overline\pi); C, \alpha, \beta)  + \rho (n-m) \alpha_\star.
\end{align*}
\end{proof}

Before we prove ~\cref{theorem:lap2}, we prove the following two technical lemmas which will be used to prove ~\cref{theorem:lap2}.
\begin{restatable}[]{lemma}{lemmaMarginalFeasibility} 
     \label{lemma marginalf}
Let \(\pi^* \in \mathcal{M}\) be a solution for the partial graph matching problem, that is \(\pi^\star \in \argmin_{\pi \in \mathcal{M}} \TCS(\pi; C, \alpha, \beta)\). For any $1 \leq p \leq m$ and any $1 \leq q \leq n$, 
\[
((\pi^*_1)_p =0) \text{ and } ((\pi^*_2)_q =0) \implies C_{pq} \geq \rho(\alpha_{p}+\beta_{q}).
\]
\end{restatable}

\begin{proof}
We prove this by contradiction and assume there exist $p,q$ such that $(\pi^*_1)_p =0$, $(\pi^*_2)_q =0$, and $C_{pq} < \rho(\alpha_{p}+\beta_{q})$.

As $(\pi^*_1)_p =0$ and $(\pi^*_2)_q =0$, $\TCS(\pi^*; C, \alpha, \beta)$ can be expanded as 
\begin{align*}
    \TCS(\pi^*; C, \alpha, \beta) 
    = \sum_{\substack{1\le i \le m \\ i \neq p}}\sum_{\substack{1 \le j \le n \\ j \neq q}} \pi^*_{ij}C_{ij}
    + \rho \left(\sum_{\substack{1\le i \le m \\ i \neq p}} \alpha_{i}(1-(\pi^*_{1})_{i})+ \alpha_{p}\right)
    \\
    +\rho \left( \sum_{\substack{1 \le j \le n \\ j \neq q}}\beta_{j}(1-(\pi^*_{2})_{j}) +  \beta_{q} \right) 
\end{align*}

Now we consider a plan $ \pi^\prime \in \mathcal{M}$ such that $\pi^\prime_{pq}=1$, but otherwise $\pi^\prime_{ij} = \pi^*_{ij}$ for all other $1 \le i \le m$ and $1 \le j \le n$. Note that this is possible as mass of $p$ and $q$ is not transported in $\pi^*$. Based on marginal constraints, we know that $\pi^\prime_{pq}=1$ implies that $\pi^\prime_{pj}=0$ for all $j \ne q$  and $\pi^\prime_{iq}=0$ for all $i \ne p$. Therefore, $\TCS(\pi^\prime; C, \alpha, \beta)$ can be rewritten as, 
\begin{align*}
    \TCS(\pi^\prime; C, \alpha, \beta) = \sum_{\substack{1\le i \le m \\ i \neq p}} \sum_{\substack{1 \le j \le n \\ j \neq q}} \pi^\prime_{ij}C_{ij} + C_{pq}
    + \rho \left(\sum_{\substack{1\le i \le m \\ i \neq p}} \alpha_{i}(1-(\pi^\prime_{1})_{i})\right)
    \\
    +\rho \left( \sum_{\substack{1 \le j \le n \\ j \neq q}} \beta_{j}(1-(\pi^\prime_{2})_{j}) \right) 
\end{align*}
Thus we see that the difference in the objective between these two plans is $ \TCS(\pi^\prime; C, \alpha, \beta)-\TCS(\pi^*; C, \alpha, \beta)= C_{pq} - \rho(\alpha_{p}+\beta_{q})$, and finally as $C_{pq} < \rho(\alpha_{p}+\beta_{q})$ we have
\[ 
\TCS(\pi^\prime; C, \alpha, \beta)<\TCS(\pi^*; C, \alpha, \beta).
\]
However, this is a contradiction as $\pi^\star \in \argmin_{\pi \in \mathcal{M}} \TCS(\pi; C, \alpha, \beta)$. Therefore, the assumption is wrong. Thus, for any \(\pi^\star \in \argmin_{\pi \in \mathcal{M}} \TCS(\pi; C, \alpha, \beta)\), 
$(\pi^*_1)_p =0$ and $(\pi^*_2)_q =0$ implies that $C_{pq} \geq \rho(\alpha_{p}+\beta_{q})$.
\end{proof}

\begin{restatable}[]{lemma}{OPT2LAP2} 
     \label{lemma opt2lap2}
Let \(\pi^* \in \mathcal{M}\) be a solution for the partial graph matching problem, that is \(\pi^\star \in \argmin_{\pi \in \mathcal{M}} \TCS(\pi; C, \alpha, \beta)\). Then,  there exists \(\overline{\pi} \in \mathcal{P}_n\) such that $\langle \overline{\pi}, \overline{C} \rangle_F = \TCS(\pi^\star; C, \alpha, \beta) + \rho (n - m) \alpha_*$. 
\end{restatable}

\begin{proof}
Let us consider some $\pi \in \argmin_{\pi \in \mathcal{M}} \TCS(\pi; C, \alpha, \beta)$. Our objective is to show that we can find $\pi' \in \mathcal{P}_n$ such that $\TCS(\alpha,\beta,\pi,C)+(n-m)\alpha^{\star}= \langle \pi^\prime,\overline{C} \rangle$ where $\overline{C} \in \mathbb{R}^{n \times n}$ is defined as \cref{eq:C_bar}

First, we extend $\pi$ to $\pi^\prime \in \{0,1\}^{n \times n}$ by padding with zeros, defining it as 
 \[
    \pi^\prime_{ij}= 
\begin{cases}
    \pi_{ij},& \text{if } i \leq m  \\
    0 & \text{if }  i > m.
\end{cases}
\]
The idea of the proof that follows is that we define a permutation matrix $\overline \pi \in \mathcal{P}_n$ that is equal to $\pi$ where it is 1, but fills in the unmatched nodes with some arbitrary permutation. That is, we assume a decomposition s.t. for all $1 \leq i,j \leq n$
\begin{equation}
\label{eq:matrixdef}
  \overline{\pi}_{ij} = \pi^\prime_{ij} + \hat{\pi}_{ij},
\end{equation}
with $\pi^\prime_{ij} = 1$ if $\pi_{ij} = 1$, 
and we note that the marginals satisfy $\overline{\pi}_1 = \pi^\prime_1 + \hat{\pi}_1 = \mathbf{1}_n$ and $\overline{\pi}_2 = \pi^\prime_2 + \hat{\pi}_2 = \mathbf{1}_n$.

Now the total cost function \cref{eq:TC} can be rewritten as
\begin{align*}
\label{eq:TC11}
    \TCS(\pi; C, \alpha, \beta) 
    &= \langle \pi, C \rangle_F + 
    \rho \left( \sum_{i=1}^m \alpha_{i} (\hat{\pi}_1)_i
    + \sum_{j=1}^n \beta_{j} (\hat{\pi}_2)_j \right ) \\
\end{align*}
By the definition of  $\pi^\prime_{ij} $, we know that for all $i>m$, $\pi^\prime_{ij} = 0$. Therefore, for all $i>m$, $\hat{\pi}_1=1$. Thus, we have,
\[
\sum_{i=m+1}^{n} (\hat{\pi}_1)_{i}=(n-m)
\]
Therefore, we can derive the following,
\begin{align}
    \TCS(\pi; C, \alpha, \beta) + \rho(n-m)\alpha_{*} 
    &= \langle \pi, C \rangle_F + 
    \rho \left( \sum_{i=1}^m \alpha_{i} (\hat{\pi}_1)_i
    + \sum_{j=1}^n \beta_{j} (\hat{\pi}_2)_j + \sum_{i=m+1}^{n} (\hat{\pi}_1)_{i}\alpha_{*} \right ) \nonumber \\
    &= \sum_{\substack{1 \le i \le m \\ 1 \le j \le n}} \pi^\prime_{ij}C_{ij} + 
     \sum_{\substack{1 \le i \le m \\ 1 \le j \le n}} \rho (\alpha_{i} 
    + \beta_{j}) \hat{\pi}_{ij} + \sum_{\substack{m < i \le n \\ 1 \le j \le n}} \rho(\beta_{j}+\alpha_{*})\hat{\pi}_{ij} \label{eq:opt2lap12}
\end{align}

 The contrapositive of \cref{lem:1} tells us that as $\pi$ is optimal, $\pi_{ij}=1$ implies that $C_{ij} \leq \rho(\alpha_{i}+\beta_{j})$. By definition of $\pi^\prime$, we know that, $\pi^\prime_{ij}=1 $ if and only $ \pi_{ij}=1$, so hence $\pi^\prime_{ij}=1$ also implies that $C_{ij} \leq \rho(\alpha_{i}+\beta_{j})$. Thus, based on the definition of $\overline{C}$ in \cref{eq:C_bar}, we can deduce that
 \begin{equation}
 \label{eq:opt2lap13}
     \sum_{\substack{1 \le i \le m \\ 1 \le j \le n}} \pi_{ij}C_{ij} = \sum_{\substack{1 \le i \le m \\ 1 \le j \le n}} \pi^\prime_{ij}C_{ij} = \sum_{\substack{1 \le i \le m \\ 1 \le j \le n}} \pi^\prime_{ij}\overline{C}_{ij}
 \end{equation}

From ~\cref{eq:matrixdef}, we know that, for any $1 \leq i,j \leq n$, if $\hat{\pi}_{ij} = 1$ then $(\pi^\prime_{1})_{i}=0$ and $(\pi^\prime_{2})_{j}=0$. Moreover, from the definition of $\pi^\prime$, for any $1 \leq i \leq m$ and $1 \leq j \leq n$, $(\pi^\prime_{1})_{i}=0$ and $(\pi^\prime_{2})_{j}=0$ if and only if $(\pi_{1})_{i}=0$ and $(\pi_{2})_{j}=0$.
And, from \cref{lemma marginalf}, we know that $(\pi_{1})_{i}=0$ and $(\pi_{2})_{j}=0$ implies that $C_{ij} \geq \rho(\alpha_{i}+\beta_{j})$. Therefore, for any $1 \leq i \leq m$ and $1 \leq j \leq n$, $\hat{\pi}_{ij} = 1$
implies that $C_{ij} \geq \rho(\alpha_{i}+\beta_{j})$ hence that $\overline{C}_{ij} = \rho(\alpha_i + \beta_j)$. Thus, we can deduce,
\begin{equation}
\label{eq:opt2lap14}
    \sum_{\substack{1 \le i \le m \\ 1 \le j \le n}} \rho (\alpha_{i} 
    + \beta_{j}) \hat{\pi}_{ij} 
    = \sum_{\substack{1 \le i \le m \\ 1 \le j \le n}}  \hat{\pi}_{ij} \overline{C}_{ij}
\end{equation}

From definition of $\overline{C}$, we know that for all $i > m$,$\overline{C}_{ij}=\rho(\alpha_{*}+\beta_{j})$. Therefore, we have,
\begin{equation}
\label{eq:opt2lap15}
    \sum_{\substack{m < i \le n \\ 1 \le j \le n}} \rho(\beta_{j}+\alpha_{*})\hat{\pi}_{ij} 
    = \sum_{\substack{m < i \le n \\ 1 \le j \le n}} \hat{\pi}_{ij} \overline{C}
\end{equation}

From ~\cref{eq:opt2lap13,eq:opt2lap14,eq:opt2lap15}, and using the fact that $\pi^\prime_{ij} = 0$ for $m < i \le n$, we can rewrite ~\cref{eq:opt2lap12} as
\begin{align}
\label{eq:opt2lap16}
    \TCS(\pi; C, \alpha, \beta) + \rho(n-m)\alpha_{*} 
    &= \sum_{\substack{1 \le i \le m \\ 1 \le j \le n}} \pi^\prime_{ij}\overline{C}_{ij} + 
     \sum_{\substack{1 \le i \le m \\ 1 \le j \le n}}  \hat{\pi}_{ij}\overline{C}_{ij} + \sum_{\substack{m < i \le n \\ 1 \le j \le n}} \hat{\pi}_{ij}\overline{C}_{ij} \nonumber \\
     &= \sum_{\substack{1 \le i \le n \\ 1 \le j \le n}} \pi^\prime_{ij}\overline{C}_{ij} + 
     \sum_{\substack{1 \le i \le n \\ 1 \le j \le n}}  \hat{\pi}_{ij}\overline{C}_{ij}  \nonumber \\
    &= \langle \overline{\pi}, \overline{C} \rangle_F 
\end{align}
\end{proof}
\theoremeqLAPthree*
\begin{proof}
Let $ \overline{\pi}^\star \in \argmin_{\overline{\pi} \in \mathcal{P}} \langle \overline{\pi}, \overline{C} \rangle_F\ $ be any optimal solution of ~\cref{eq:LAP} and  $\pi^\star \in \argmin_{\pi \in \mathcal{M}}\TCS(\pi; C, \alpha,\beta)$ be any optimal solution of ~\cref{eq:UGM}

From ~\cref{lemma opt2lap2}, we know that,  there exists a $\overline{\pi} \in  \mathcal{P}_n $ s.t. $\TCS(\pi^\star; C, \alpha,\beta)+(n-m)\alpha_* = \langle \overline{\pi},\overline{C} \rangle$. Therefore,
\begin{equation}
\label{eq:optimalinequality1}
    \langle \overline{\pi}^{*},\overline{C} \rangle 
    \leq \langle \overline{\pi},\overline{C} \rangle 
    =\TCS(\pi^\star; C,\alpha,\beta)+\rho(n-m)\alpha_*
\end{equation}

From ~\cref{theorem lap2}, we have $h(\overline{\pi}^*) \in \mathcal{M}$ as defined in \cref{eq:h} from $\overline{\pi}^*$ such that the following condition holds,
\begin{equation}
\label{eq:optimalequality}
    \TCS(h(\overline{\pi}^*);C,\alpha,\beta)+\rho(n-m)\alpha_* = \langle \overline{\pi}^*,\overline{C} \rangle
\end{equation}
Finally, from ~\cref{eq:optimalinequality1} and ~\cref{eq:optimalequality},
\[
\TCS(h(\overline{\pi}^*);C,\alpha,\beta)+\rho (n-m)\alpha_*\leq   \TCS(\pi^\star;C,\alpha,\beta)+\rho (n-m)\alpha_* 
\]
and as $\pi^\star$ is already a minimizer of $\TCS$, implies that 
$\TCS(h(\overline{\pi}^*);C,\alpha,\beta) =   \TCS(\pi^\star;C,\alpha,\beta)$,
which indicates that $h(\overline{\pi}^*) \in \argmin_{ \pi \in \mathcal{M}} \TCS(\pi;C,\alpha,\beta)$. 

Thus, from any $\overline{\pi}^{*} \in \argmin_{\overline{\pi}' \in \mathcal{P}_{n}} \langle \overline{\pi}, \overline{C} \rangle_F\ $, which is an optimal solution of ~\cref{eq:LAP}, it is possible to derive $h(\overline{\pi}^*) \in \argmin_{ \pi \in \mathcal{M}} \TCS(\pi;C,\alpha,\beta)$, which solves ~\cref{eq:PGM}). Thus, the proof is complete.
\end{proof}

\section{Additional Details on Experiments}

\begin{table*}[t!]
    \centering
     \resizebox{\textwidth}{!}
    {\renewcommand{\arraystretch}{1.2}
\begin{tabular}{l l| c c c c c c c c c c c c c c c c c c c c |c} 
 \toprule
 GM Network & PMH & aero &	bike & bird	& boat	& bottle & bus & car & cat & chair & cow & table & dog & horse & mbike & person & plant & sheep	& sofa	& train	& tv & mean \\ 
 \toprule
NGM-v2 & dummy & 44.7 & 61.9 & 57.1 & 41.9 & 83.9 & 63.9 & \textbf{\textcolor{blue(pigment)}{54.1}} & 60.8 & 40.5 & 64.2 & 36.2 & \textbf{60.6} & 60.8 & 61.9 & \textbf{\textcolor{blue(pigment)}{48.7}} & 91.2 & \textbf{56.2} & 37.4 &63.2 &82.2 & 58.6$\pm$0.5 \\
NGM-v2 & AFAT-U & 45.7 & 67.7 & 57.3 & 44.9& 90.1 & 65.5 & 49.9 & 59.3 & 44.0& 62.0 & \textbf{54.9} & 58.4 & 58.6 & 63.8 & 45.9 & 94.8 & 50.9 & 37.3 & \textbf{\textcolor{blue(pigment)}{74.2}} & 82.8 & 60.2$\pm$0.4\\
NGM-v2 & AFAT-I & 45.0  & 67.3  & 55.9  & 45.6  & 90.3 & 64.6 & 48.7 & 58.0 & 44.7 & 60.2 & \textbf{\textcolor{blue(pigment)}{54.8}} & 57.2 & 57.5 & 63.4 & 45.2 & 95.3 & 49.3 & 41.6 & 73.6 & 82.4 & 59.9$\pm$0.3 \\ 

GCAN & ILP & 46.3 & 67.7 & 57.4 & 45.0 & 87.1 & 64.8 & \textbf{57.5} & \textbf{\textcolor{blue(pigment)}{61.2}} & 40.8 & 61.6 & 37.3 & 59.9 & 59.2 & 64.6 & \textbf{49.7} & 95.1 & \textbf{\textcolor{blue(pigment)}{54.5}} & 28.5 & \textbf{77.9} & \textbf{\textcolor{blue(pigment)}{83.1}} & 59.7$\pm$0.3\\
GCAN & AFAT-U & \textbf{47.1} & \textbf{\textcolor{blue(pigment)}{70.8}} & \textbf{58.1} & 45.8 & \textbf{90.8} & \textbf{66.5} & 49.6 & 58.8 & \textbf{50.6} & \textbf{\textcolor{blue(pigment)}{64.6}} & 47.2 & \textbf{\textcolor{blue(pigment)}{60.5}} & \textbf{62.3} & 65.7 & 46.3 & 95.4 & 52.7 & \textbf{\textcolor{blue(pigment)}{47.4}} & \textbf{\textcolor{blue(pigment)}{74.2}} & \textbf{83.8} & \textbf{62.0}$\pm$\textbf{0.2}\\
GCAN & AFAT-I & 46.1 & 69.9 & 56.1 & 46.6 & \textbf{\textcolor{blue(pigment)}{90.7}} & \textbf{\textcolor{blue(pigment)}{66.1}} & 48.1 & 57.9 & 49.9 & 63.9 & 50.4 & 59.0 & \textbf{\textcolor{blue(pigment)}{61.6}} & 65.0 & 44.7 & \textbf{\textcolor{blue(pigment)}{95.5}} & 50.9 & \textbf{49.2} & 74.0 & \textbf{83.8} & \textbf{\textcolor{blue(pigment)}{61.6}}$\pm$ \textbf{\textcolor{blue(pigment)}{0.3}}\\
 \hline
GCAN & OPGM & \textbf{\textcolor{blue(pigment)}{46.5}}	& 70.1	& 55.8	& \textbf{\textcolor{blue(pigment)}{47.1}}	& 89.5	& 62.4	& 46 & 60.8 &	48.9	& 63.4	& 46.2	& 58.9	& 60.2	& \textbf{\textcolor{blue(pigment)}{67.7}}	& 47.7	& 94.9	& 51.5	& 42	& 73.1	& 82.9	&60.8 $\pm$ 0.4  \\ 
GCAN & OPGM-rs & \textbf{47.1} & \textbf{71.5} & \textbf{\textcolor{blue(pigment)}{57.5}} & \textbf{47.9} & 89.6 & 64.1 & 46.8 & \textbf{62.3} & \textbf{\textcolor{blue(pigment)}{50.0}} & \textbf{65.4} & 33.2 & 59.4 & 61.2 & \textbf{68.7} & 48.5 & \textbf{95.6} & 53.4 & 39.8 & 74.1 & \textbf{\textcolor{blue(pigment)}{83.1}} & 61.0 $\pm$ 0.2\\ 
\bottomrule
\end{tabular}%
}
\caption{Matching F1-score on Pascal VOC Keypoint. The best results are colored in \textbf{black} and the second best are in \textcolor{blue(pigment)}{\textbf{blue}}}.
\label{tab:pascalVOC}
\end{table*}
\begin{table}[] 
    \centering
    \resizebox{0.6\textwidth}{!}{
    \begin{tabular}{l| c c c c c c c c|c } 
 \toprule
 Method & Yeast 5\% & Yeast 10\% & Yeast 15\%& Yeast 20\% & Yeast 25\%  \\ 
 \toprule

 SIGMA & 84.7$\pm${0.4} & 68.8 $\pm$ 2.5 &57.4$\pm${1.1}&46.7 $\pm${2.3} & 41.4 $\pm${1.7} \\

StableGM & 86.1 $\pm$ 0.9 & 75.6 $\pm$ 0.8& 67.9 $\pm$ 1.1 &  63.2 $\pm$ 0.9 &57 $\pm$ 0.6\\ 
\hline
OPGM (ours) & \textbf{\textcolor{blue(pigment)}{87.8}} $\pm$ \textbf{\textcolor{blue(pigment)}{0.3}}  & \textbf{80} $\pm$ \textbf{0.4} & \textbf{71.9} $\pm$ \textbf{0.9} &  \textbf{66.9} $\pm$ \textbf{1.0}&\textbf{58.8} $\pm$ \textbf{0.8}\\
OPGM-rs (ours) & \textbf{88.3} $\pm$ \textbf{0.5} & \textbf{\textcolor{blue(pigment)}{79.8}} $\pm$ \textbf{\textcolor{blue(pigment)}{0.6}} & \textbf{\textcolor{blue(pigment)}{71.7}} $\pm$ \textbf{\textcolor{blue(pigment)}{0.8}} & \textbf{\textcolor{blue(pigment)}{66.2}} $\pm$ \textbf{\textcolor{blue(pigment)}{0.9}} &\textbf{\textcolor{blue(pigment)}{57.7}} $\pm$ \textbf{\textcolor{blue(pigment)}{1.0}}\\
\bottomrule
\end{tabular}
}
    \caption{Node correctness (\%) results on the PPI dataset. The best results are colored in \textbf{black} and the second best are in \textcolor{blue(pigment)}{\textbf{blue}}}.
    
    \label{tab:PPI}
\end{table}

\paragraph{Discussion on results related to PascalVOC and PPI Network Matching}
When observing the results given in ~\cref{tab:pascalVOC}, it can be seen that the mean matching F1-Scores of OPGM, and OPGM-rs are less than that of AFAT-U (GCAN) and AFAT-I (GCAN) models. Furthermore, in the Pascal VOC Keypoint dataset, OPGM performs better than OPGM-rs in terms of  mean matching F1-Score, primarily due to the poor performance of OPGM-rs in the \textit{table} and \textit{sofa} classes. It is important to note that these two classes were removed from the SPair-71K dataset due to ambiguous and poor annotations~\citep{rolinek2020deep}. Moreover, In PPI network matching~\cref{tab:PPI}, at lower noise levels, OPGM-rs tends to outperform OPGM, but as noise increases, OPGM performs better. In OPGM-rs, the learnable cost matrix \( C \) is used in the learning process. However, with noisy inputs, the information in \( C \) can also become noisy, and simultaneously learning \( \alpha \) and \( \beta \) alongside \( C \) may negatively impact the training process as noise levels rise. So based on this observation and the observations related to PPI network matching, we can observe that the performance of OPGM-rs is adversely affected by poor/ambiguous annotations and increased noise in data.

\paragraph{Discussion on efficiency analysis} 
As discussed in ~\cref{sec:experiments}, all the models considered for analysis shown in ~\cref{fig:combined_performance_and_inference} (right) uses the same neural architecture ~\citep{jiang2022graph} to obtain the cross-graph node-to-node affinity matrix and the use Sinkhorn normalization to obtain the doubly stochastic affinity matrix.
Once these matrices are obtained, our model OPGM-rs solves the partial graph matching problem as described in ~\cref{sec:solvingPGM} by solving a linear sum assignment problem. We use the linear sum assignment problem solver provided by SciPy`\citep{virtanen2020scipy} in our implementation. GCAN~\citep{jiang2022graph} solves partial graph matching as an ILP problem, which is usually computationally expensive than solving a LAP. 

In the experimental setting we consider \citep{wang2023deep}, OR-tools, which is the Google's open source software suite for combinatorial optimization has been used to solve the ILP problem. On the other hand, once obtaining the doubly stochastic affinity matrix, AFAT-I and AFAT-U modules use seperate neural modules to first obtain the number of matchings that exist between the two graphs that includes solving an entropy regularized optimal transport problem. Finally, GreedyTopK algorithm~\citep{wang2023deep} which is based on Hungarian algorithm, is used to obtain the final matching. The additional complexities added by separate modules has resulted in higher average inference time of AFAT-I and AFAT-U models, when compared to OPGM-rs. 

The experiments on efficiency analysis were conducted on a Linux server with an Intel Xeon W-2175 2.50GHz processor alongside 28 cores, NVIDIA RTX A6000 GPU, and 512GB of main memory.

\paragraph{Hyperparameters}

For the image keypoint matching task, the hyper-parameters of OPGM-rs are searched in the following range:  $\rho$ $\in \{0.1 ,0.2, 0.3, 0.35, 0. 4, 0.5\}$, $\lambda$ $\in \{0.01,0.2,0.3,0.4,0.5,0.6,$ $0.8,1.0\}$, learning rate $\in \{0.001, 0.002\}$, VGG16 backbone learning rate  $\in \{0.0001\}$, batch size $\in \{4,8\}$,  and the number of epochs $\in \{15, 20, 25\}$. We use the Adam algorithm \citep{diederik2014adam} as our optimizer. The initial learning rate decays with a factor of 0.5 after every 2 epochs.
For the PPI network matching task, the hyper-parameters of OPGM-rs are searched in the following ranges:  $\rho = 10^{11}$, $\lambda \in \{0.25, 0.5, 0.75, 1\}$,  learning rate $\in \{0.0001, 0.0002\}$,  and the number of epochs = 100. Adam algorithm \citep{diederik2014adam} was used as the optimizer.\looseness=-1

\pgfplotstableread[col sep=comma,]{error_data.csv}\datatable

\pgfplotsset{every axis/.append style={
                    label style={font=\small},
                    tick label style={font=\scriptsize}  
                    }}

\begin{figure}[htbp]
\centering
\begin{tikzpicture}
\begin{axis}[
    width=0.4\textwidth,
    ybar,
    bar width=5pt,
    xlabel={\textbf{Datasets}},
    ylabel={\textbf{Error Rate (\%)}},
    xtick=data,
    xticklabels from table={\datatable}{Dataset},
    xticklabel style={rotate=40, anchor=north east, xshift=10pt},
    ymajorgrids=true,
    grid style={dotted, gray!30},
    ymin=0,
    ymax=50,
    legend style={
        at={(0.5,1.05)},
        anchor=south,
        legend columns=3,
        cells={anchor=west},
        font=\small,
        draw=none,
        fill=white,
        opacity=0.9
    },
    enlarge x limits={abs=0.5},       
    xtick distance=1     
]
    \addplot[fill=color1, draw=color1!80] table [x expr=\coordindex, y=Mismatching]{\datatable};
    \addplot[fill=color2, draw=color2!80] table [x expr=\coordindex, y=Partiality]{\datatable};
    \addplot[fill=color3, draw=color3!80] table [x expr=\coordindex, y=Total]{\datatable};

    \legend{\textbf{Mismatching}, \textbf{Partiality}, \textbf{Total}}
\end{axis}
\end{tikzpicture}
\caption{Failure mode analysis across four different datasets.}\vspace{-0.4cm} 
\label{fig:error-analysis}
\end{figure}

\paragraph{Failure mode analysis}

The feasibility of matching a pair of nodes \(i\) and \(j\) is influenced by their cost \(C_{ij}\), the matching biases \(\alpha_i\) and \(\beta_j\), and the hyperparameter \(\rho\). A specific failure mode we observed is when node pairs that should be matched according to the ground truth are deemed infeasible by the optimization objective in Eq. (9) (as described in Theorem 5.1). Consequently, these nodes remain unmatched in the final solution. This failure mode typically arises when the graph matching network \(\textsc{NN}_\theta\) assigns an excessively low affinity (resulting in a high matching cost) to pairs of nodes that should otherwise be matched. 

We quantify the error by analyzing:
\begin{itemize}
    \item \emph{Partiality error:} The percentage of true matches (as per the ground truth) that fail to meet the feasibility condition outlined in Theorem 5.1. This error reflects mismatches caused by incorrect partiality identification involving matching biases.

\item \emph{Mismatching error:} The percentage of true matches that are not identified as true matches even though they meet the feasibility condition.
\end{itemize}

Figure~\ref{fig:error-analysis} demonstrates the partiality error, mismatching error and total error related to this failure mode across four different datasets. The total error is equal to the summation of partiality error and the mismatching error.

\end{document}